\documentclass{article}
\usepackage[utf8]{inputenc}
\usepackage[english]{babel}
\usepackage{amsmath, amsfonts, amssymb, amsthm}
\usepackage{graphicx}
\usepackage[left=1.45in,right=1.45in,top=2.5cm,bottom=2.5cm]{geometry}
\usepackage{subcaption}

\usepackage{hyperref}
\usepackage[]{natbib}
\usepackage[utf8]{inputenc} % allow utf-8 input
\usepackage[T1]{fontenc}    % use 8-bit T1 fonts

\usepackage{algorithm,algorithmic}
\usepackage{hyperref}       % hyperlinks
\usepackage{url}            % simple URL typesetting
\usepackage{booktabs}       % professional-quality tables
\usepackage{amsfonts}       % blackboard math symbols
\usepackage{nicefrac}       % compact symbols for 1/2, etc.
\usepackage{microtype}      % microtypography
\usepackage{amsmath,amsthm,amsfonts,amssymb,mathtools}
\usepackage{graphicx}
\usepackage[usenames, dvipsnames]{color}
\usepackage{enumitem}
\usepackage{multirow}

\newcommand{\Rbb}{\mathbb{R}}

\newcommand{\Fcal}{\mathcal{F}}

\newcommand{\Ocal}{\mathcal{O}}

\newcommand{\Scal}{\mathcal{S}}

\newcommand{\Xcal}{\mathcal{X}}
\newcommand{\Ycal}{\mathcal{Y}}

\DeclareMathOperator*{\Exp}{\mathbb{E}}

\newcommand\emp[1]{{\color{Blue} #1}}

\newcommand \bra{\left\langle}
\newcommand \ket{\right\rangle}
\newcommand \braket[2]{\bra #1, #2 \ket}

\newcommand\xf{x}
\newcommand\yf{y}
\newcommand\ux{u}
\newcommand\uy{v}
\newcommand\Gx{G_u}
\newcommand\Gy{G_v}
\newcommand\hatGx{\hat G_u}
\newcommand\hatGy{\hat G_v}
\newcommand\hatux{\hat u}
\newcommand\hatuy{\hat v}
\newcommand\Ux{U}
\newcommand\Uy{V}
\newcommand\hatg{\hat g}

\usepackage{xspace}

\newcommand\fid{\emph{id}\xspace}
\newcommand\fngrams{\emph{ngrams}\xspace}
\newcommand\fcats{\emph{cats}\xspace}
\newcommand\rankone[1]{#1 \otimes #1}

\newtheorem{proposition}{Proposition}

\newcommand\email[1]{\texttt{\footnotesize #1}}
\title{Efficient Training on Very Large Corpora\\
via Gramian Estimation}

\author{
Walid Krichene\thanks{Google Research.}\\\email{walidk@google.com}\vspace{.1in}\\Xinyang Yi\footnotemark[1]\\\email{xinyang@google.com}
\and
Nicolas Mayoraz\footnotemark[1]\\\email{nmayoraz@google.com}\vspace{.1in}\\Lichan Hong\footnotemark[1]\\\email{lichan@google.com}
\and
Steffen Rendle\footnotemark[1]\\\email{srendle@google.com}\vspace{.1in}\\Ed Chi\footnotemark[1]\\\email{edchi@google.com}
\and
Li Zhang\footnotemark[1]\\\email{liqzhang@google.com}\vspace{.1in}\\John Anderson\footnotemark[1]\\\email{janders@google.com}
}
\date{}

\begin{document}
% ==================================================
\maketitle

\begin{abstract}
We study the problem of learning similarity functions over very large corpora using neural network embedding models. These models are typically trained using SGD with sampling of random observed and unobserved pairs, with a number of samples that grows quadratically with the corpus size, making it expensive to scale to very large corpora.
We propose new efficient methods to train these models without having to sample unobserved pairs. Inspired by matrix factorization, our approach relies on adding a global quadratic penalty to all pairs of examples and expressing this term as the matrix-inner-product of two generalized Gramians. We show that the gradient of this term can be efficiently computed by maintaining estimates of the Gramians, and develop variance reduction schemes to improve the quality of the estimates. We conduct large-scale experiments that show a significant improvement in training time and generalization quality compared to traditional sampling methods.
\end{abstract}

%%%%%%%%%%%%%%%%%%%%%%%%%%%%%%%%%%%%%%%%%%%%%%%%%%%%%%%%%%%%%%%%%%%%%%%%%%%%%%%
\section{Introduction}
\label{sec:introduction}
%%%%%%%%%%%%%%%%%%%%%%%%%%%%%%%%%%%%%%%%%%%%%%%%%%%%%%%%%%%%%%%%%%%%%%%%%%%%%%%
We consider the problem of learning a similarity function $h: \Xcal \times \Ycal \to \mathbb R$, that maps each pair of items, represented by their feature vectors $(\xf, \yf) \in \Xcal \times \Ycal$, to a real number $h(\xf, \yf)$, representing their similarity. We will refer to $\xf$ and $\yf$ as the left and right feature vectors, respectively. Many problems can be cast in this form: In a natural language processing setting, $\xf$ represents a context (e.g. a bag of words), $\yf$ represents a candidate word, and the target similarity measures the likelihood to observe $\yf$ in context $\xf$~\citep{mikolov2013word2vec,pennington2014glove,levy2014neural}. In recommender systems, $\xf$ represents a user query (the user id and any available contextual information), $\yf$ represents a candidate item to recommend, and the target similarity is a measure of relevance of item $y$ to query $x$, e.g. a movie rating~\citep{agarwal2009regression}, or the likelihood to watch a given movie~\citep{hu2008collborative,rendle2010FM}. %In these examples, the left and right feature spaces $\Xcal$ and $\Ycal$ are different. In other domains, the feature spaces can coincide. For example, 
Other applications include image similarity, where $\xf$ and $\yf$ are pixel-representations of a pair of images~\citep{bromley1993signature,chechik2010large,schroff2015facenet}, and network embedding models~\citep{grover2016node2vec, qiu2018network}, where $\xf$ and $\yf$ are nodes in a network and the target similarity is wheter an edge connects them.

%There is a rich literature on learning similarity functions, see~\citep{kullis2013metric} for a survey. Some approaches rely on explicitly estimating a partially-observed distance matrix, such as the seminal work of~\cite{xing2002distance}, but in these formulations, the decision variable is quadratic in the corpus size, which makes it impractical for problems with very large corpora.

A popular approach to learning similarity functions is to train an embedding representation of each item, such that items with high similarity are mapped to vectors that are close in the embedding space. A common property of such problems is that only a very small subset of all possible pairs $\Xcal \times \Ycal$ is present in the training set, and those examples typically have high similarity. Training exclusively on observed examples has been demonstrated to yield poor generalization performance. Intuitively, when trained only on observed pairs, the model places the embedding of a given item close to similar items, but does not learn to place it far from dissimilar ones~\citep{shazeer2016swivel,xin2017folding}.

Taking into account unobserved pairs is known to improve the embedding quality in many applications, including recommendation~\citep{hu2008collborative,yu2017selection} and word analogy tasks~\citep{shazeer2016swivel}. This is often achieved by adding a low-similarity prior on \emph{all pairs}, which acts as a repulsive force between all embeddings. But because it involves a number of terms quadratic in the corpus size, this term is computationally intractable (except in the linear case), and it is typically optimized using sampling: for each observed pair in the training set, a set of random unobserved pairs is sampled and used to compute an estimate of the repulsive term. But as the corpus size increases, the quality of the estimates deteriorates unless the sample size is increased, which limits scalability. In this paper, we address this issue by developing new methods to efficiently estimate the repulsive term without having to sample a large number of unobserved pairs.

%%%%%%%%%%%%%%%%%%%%%%%%%%%%%%%%%%%%%%%%%%%%%%%%%%%%%%%%%%%%%%%%%%%%%%%%%%%%%%%
\subsubsection*{Related work}
Our approach is inspired by matrix factorization models, which correspond to the special case of linear embedding functions. They are typically trained using alternating least squares~\citep{hu2008collborative}, or coordinate descent methods~\citep{bayer2017generic}, which circumvent the computational burden of the repulsive term by writing it as a matrix-inner-product of two Gramians, and computing the left Gramian before optimizing over the right embeddings, and vice-versa.

Unfortunately, in non-linear embedding models, each update of the model parameters induces a simulateneous change in all embeddings, making it impractical to recompute the Gramians at each iteration. As a result, the Gramian formulation has been largely ignored in the non-linear setting. Instead, non-linear embedding models are trained using stochastic gradient methods with sampling of unobserved pairs, see~\cite{chen2016strategies}. In its simplest variant, the sampled pairs are taken uniformly at random, but more sophisticated schemes have been proposed, such as adaptive sampling~\citep{bengio2008adaptive,bai2017tapas}, and importance sampling~\citep{bengio2003quick,mikolov2013word2vec} to account for item frequencies. % to put more weight on pairs of examples that are close in the embedding space.
We also refer to~\cite{yu2017selection} for a comparative study of sampling methods in recommender systems. \cite{vincent2015efficient} were, to our knowledge, the first to attempt leveraging the Gramian formulation in the non-linear case. They consider a model where \emph{only one of the embedding functions is non-linear}, and show that the gradient can be computed efficiently in that case. Their result is remarkable in that it allows exact gradient computation, but this unfortunately does not generalize to the case where both embedding functions are non-linear.

%%%%%%%%%%%%%%%%%%%%%%%%%%%%%%%%%%%%%%%%%%%%%%%%%%%%%%%%%%%%%%%%%%%%%%%%%%%%%%%
\subsubsection*{Our contributions}
We propose new methods that leverage the Gramian formulation in the non-linear case, and that, unlike previous approaches, are efficient even when both left and right embeddings are non-linear. Our methods operate by maintaining stochastic estimates of the Gram matrices, and using different variance reduction schemes to improve the quality of the estimates. Perhaps most importantly, they do not require sampling large numbers of unobserved pairs, and experiments show that they scale far better than traditional sampling approaches when the corpus is very large. %The first is inspired by the SAG(A) methods~\citep{schmidt2017SAG,defazio2014saga}, and relies on caching embeddings during training. The second reformulates the problem as a two-player game, in which the first player optimizes over the model parameters while the second optimizes over Gramians.

We start by reviewing preliminaries in Section~\ref{sec:preliminaries}, then derive the methods and analyze them in Section~\ref{sec:estimation}. Finally, we conduct large-scale experiments in Section~\ref{sec:experiments}, on a classification task on the Wikipedia dataset and a regression task on the MovieLens dataset. All the proofs are deferred to the appendix.

%%%%%%%%%%%%%%%%%%%%%%%%%%%%%%%%%%%%%%%%%%%%%%%%%%%%%%%%%%%%%%%%%%%%%%%%%%%%%%%
\section{Preliminaries}
\label{sec:preliminaries}

\begin{figure}
\centering
\includegraphics[width=.7\textwidth]{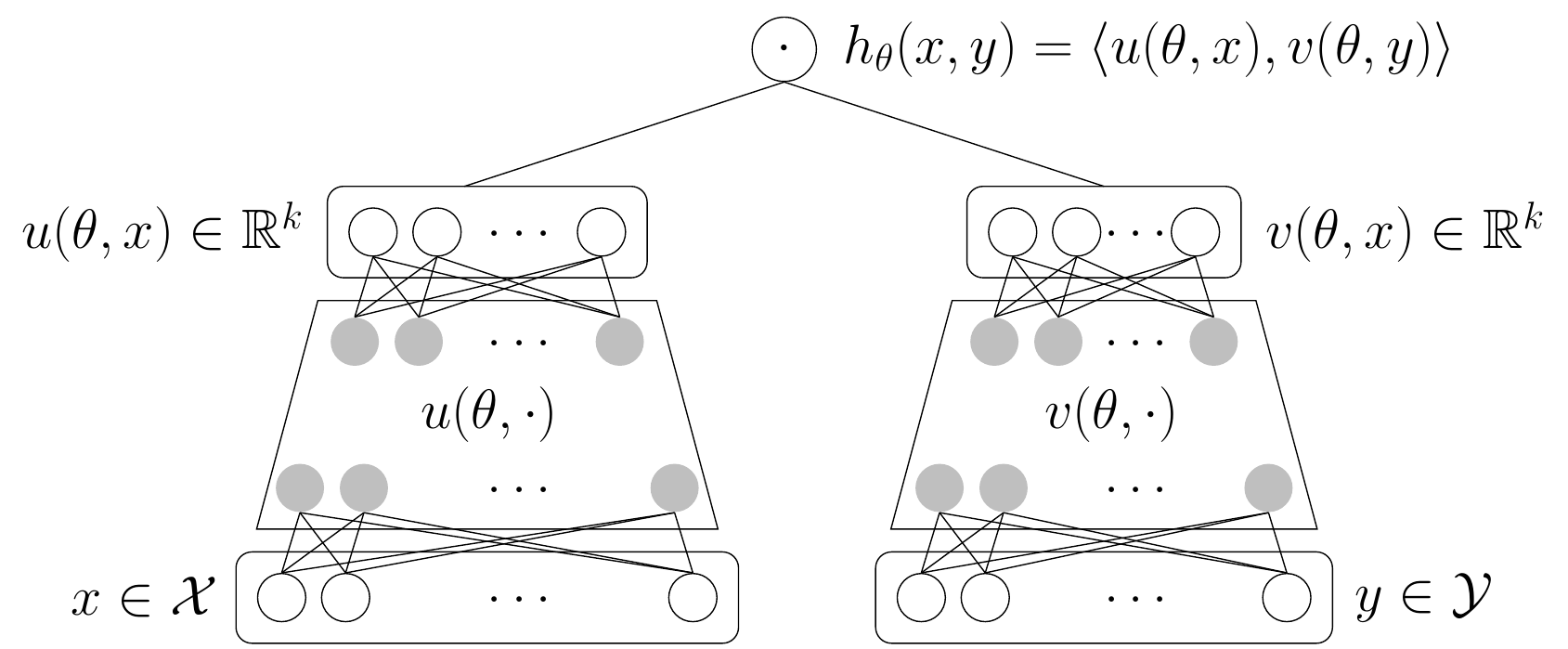}
\caption{An inner-product embedding model for learning a similarity function on $\Xcal \times \Ycal$. %The left and right feature vectors are mapped to a common embedding space $\Rbb^k$, and the output of the model is the dot product of embeddings.
}
\label{fig:model}
\end{figure}

\subsection{Notation and problem formulation}
We consider embedding models that consist of two embedding functions $u: \Rbb^d \times \Xcal \to \Rbb^k$ and $v: \Rbb^d \times \Ycal \to \Rbb^k$, which map a parameter vector\footnote{In many applications, it is desirable for the two embedding functions $u, v$ to share certain parameters, e.g. embeddings of categorical features common to left and right items; hence, we use the same $\theta$ for both.} $\theta \in \Rbb^d$ and feature vectors $x, y$ to embeddings $u(\theta, x), v(\theta, y) \in \Rbb^k$. The output of the model is the inner product\footnote{This also includes cosine similarity models when the embedding functions $u, v$ are normalized.}\footnote{One advantage of an inner-product model is that it allows for efficient retrieval: given a query item $x$, the problem of retrieving items $y$ with high similarity to $x$ is a maximum inner product search problem (MIPS), which can be approximated efficiently~\citep{shrivastava2014ALSH, neyshabur2015symmetric}.} of the embeddings
\begin{equation}
\label{eq:model}
h_\theta(\xf, \yf) = \braket{u(\theta, \xf)}{v(\theta, \yf)},
\end{equation}
where $\braket{\cdot}{\cdot}$ denotes the usual inner-product on $\Rbb^k$. Low-rank matrix factorization is a special case of~\eqref{eq:model}, in which the left and right embedding functions are linear in $x$ and $y$. Figure~\ref{fig:model} illustrates a non-linear model, in which each embedding function is given by a feed-forward neural network. We denote the training set by
\[
T = \{(\xf_i, \yf_i, s_i) \in \Xcal \times \Ycal \times \Rbb\}_{i \in \{1, \dots, n\}},
\]
where $\xf_i, \yf_i$ are the feature vectors and $s_i$ is the target similarity for example $i$. To make notation more compact, we will use $u_i(\theta), v_i(\theta)$ as a shorthand for $u(\theta, x_i), v(\theta, y_i)$, respectively.

As discussed in the introduction, we also assume that we are given a low-similarity prior $p_{ij} \in \Rbb$ for all pairs $(i, j) \in \{1, \dots, n\}^2$. Given a scalar loss function $\ell: \Rbb \times \Rbb \to \Rbb$, the objective function is given by
\begin{equation}
\label{eq:objective}
\min_{\theta \in \Rbb^d} \frac{1}{n} \sum_{i = 1}^n \ell\left(\braket{u_i(\theta)}{v_i(\theta)}, s_i\right) + \frac{\lambda}{n^2} \sum_{i = 1}^n \sum_{j = 1}^n (\braket{u_i(\theta)}{v_j(\theta)} - p_{ij})^2,
\end{equation}
where $\lambda$ is a positive hyper-parameter. To simplify the discussion, we will assume a uniform zero prior $p_{ij}$ as in~\citep{hu2008collborative}, but we relax this assumption in Appendix~\ref{app:low-rank}.

The last term in~\eqref{eq:objective} is a double-sum over the training set and can be problematic to optimize efficiently. We will denote it by
\[
g(\theta) \coloneqq \frac{1}{n^2}\sum_{i = 1}^n \sum_{j = 1}^n  \braket{u_i(\theta)}{v_j(\theta)}^2.
\]
Existing methods typically rely on sampling to approximate $g(\theta)$, and are usually referred to as negative sampling or candidate sampling, see~\cite{chen2016strategies,yu2017selection} for recent surveys. %This can be done for example by sampling a batch $B \subset \{1, \dots, n\}$, and approximating $g_i(\theta)$ by
%\begin{equation}
%\label{eq:negative_sampling}
%\frac{1}{2|B|} \sum_{j \in B} \braket{u_i(\theta)}{v_j(\theta)}^2 + \braket{u_j(\theta)}{v_i(\theta)}^2.
%\end{equation}
Due to the double sum, the quality of the sampling estimates degrades as the corpus size increases, which can significantly increase training times. This can be alleviated by increasing the sample size, but does not scale to very large corpora.

%%%%%%%%%%%%%%%%%%%%%%%%%%%%%%%%%%%%%%%%%%%%%%%%%%%%%%%%%%%%%%%%%%%%%%%%%%%%%%%
\subsection{Gramian formulation}
\label{sec:gram_formulation}
A different approach to optimizing~\eqref{eq:objective}, widely popular in matrix factorization, is to rewrite $g(\theta)$ as the inner product of two Gram matrices. Let us denote by $U_\theta \in \Rbb^{n \times k}$ the matrix of all left embeddings such that $u_i(\theta)$ is the $i$-th row of $U_\theta$, and similarly for $V_\theta \in \Rbb^{n \times k}$. Then denoting the matrix inner-product by $\braket{A}{B} = \sum_{i, j} A_{ij}B_{ij}$, we can rewrite $g(\theta)$ as:
\begin{equation}
g(\theta)
= \frac{1}{n^2}\sum_{i = 1}^n \sum_{j = 1}^n  \braket{u_i(\theta)}{v_j(\theta)}^2 = \frac{1}{n^2}\sum_{i = 1}^n \sum_{j = 1}^n (U_\theta V^\top_\theta)_{ij}^2 = \frac{1}{n^2}\braket{U_\theta V^\top_\theta}{U_\theta V^\top_\theta}. \label{eq:gravity_sum}
\end{equation}
Now, using the adjoint property of the inner product, we have $\braket{U_\theta V^\top_\theta}{U_\theta V^\top_\theta} = \braket{U^\top_\theta U_\theta}{V^\top_\theta V_\theta}$, and if we denote by $u \otimes u$ the outer product of a vector $u$ by itself, and define the Gram matrices\footnote{Note that a given left item $x$ may appear in many example pairs (and similarly for right items), one can define the Gram matrices as a sum over unique items. The two formulations are equivalent up to reweighting of the embeddings.}
\begin{align}
\label{eq:gramian_definition}
\begin{cases}
\Gx(\theta) \coloneqq \frac{1}{n}U_\theta^\top U_\theta = \frac{1}{n} \sum_{i = 1}^n \ux_i(\theta) \otimes \ux_i(\theta),\\
\Gy(\theta) \coloneqq \frac{1}{n} V^\top_\theta V_\theta = \frac{1}{n}\sum_{i = 1}^n \uy_i(\theta) \otimes \uy_i(\theta),
\end{cases}
\end{align}
we have
\begin{equation}
\label{eq:gravity_dot}
g(\theta) = \braket{\Gx(\theta)}{\Gy(\theta)}.
\end{equation}
The Gramians are $k \times k$ PSD matrices, where $k$, the dimension of the embedding space, is much smaller than $n$ -- typically $k$ is smaller than $1000$, while $n$ can be arbitrarily large. Thus, the Gramian formulation~\eqref{eq:gravity_dot} has a much lower computational complexity than the double sum formulation~\eqref{eq:gravity_sum}, and this transformation is at the core of alternating least squares and coordinate descent methods~\citep{hu2008collborative, bayer2017generic}, which operate by computing the exact Gramian for one side, and solving for the embeddings on the other. However, these methods do not apply in the non-linear setting due to the dependence on~$\theta$, as a change in the model parameters simultaneously changes all embeddings, making it intractable to recompute the Gramians at each iteration, so the Gramian formulation has not been used when training non-linear models. In the next section, we will show that it can in fact be leveraged in the non-linear case, and leads to significant speed-ups in numerical experiments.

%%%%%%%%%%%%%%%%%%%%%%%%%%%%%%%%%%%%%%%%%%%%%%%%%%%%%%%%%%%%%%%%%%%%%%%%%%%%%%%
%%%%%%%%%%%%%%%%%%%%%%%%%%%%%%%%%%%%%%%%%%%%%%%%%%%%%%%%%%%%%%%%%%%%%%%%%%%%%%%
\section{Training Embedding Models using Gramian Estimates}
\label{sec:estimation}
Using the Gramians defined in~\eqref{eq:gramian_definition}, the objective function~\eqref{eq:objective} can be rewritten as a sum over examples $\frac{1}{n}\sum_{i = 1}^n [f_i(\theta) + \lambda g_i(\theta)]$, where
\begin{align}
f_i(\theta) &\coloneqq \ell\left(\braket{u_i(\theta)}{v_i(\theta)}, s_i\right)\\
g_i(\theta) &\coloneqq \frac{1}{2n} \sum_{j = 1}^n \big[\braket{u_i(\theta)}{v_j(\theta)}^2 + \braket{u_j(\theta)}{v_i(\theta)}^2\big] \notag \\
&= \frac{1}{2}[\braket{u_i(\theta)}{G_v(\theta)u_i(\theta)} + \braket{v_i(\theta)}{G_u(\theta)v_i(\theta)}].\label{eq:gravity}
\end{align}
Intuitively, for each example $i$, $-\nabla f_i(\theta)$ pulls the embeddings $u_i$ and $v_i$ close to each other (assuming a high similarity $s_i$), while $-\nabla g_i(\theta)$ creates a repulsive force between $u_i$ and all embeddings $\{v_j\}_{j \in \{1, \dots, n\}}$, and between $v_i$ and all embeddings $\{u_j\}_{j \in \{1, \dots, n\}}$. Due to this interpretation, we will refer to $g(\theta) = \sum_{i = 1}^n g_i(\theta)$ as the \emph{gravity} term, as it pulls the embeddings towards certain regions of the embedding space. We further discuss its properties and interpretations in Appendix~\ref{app:gravity}.

We start from the observation that, while the Gramians are expensive to recompute at each iteration, we can maintain PSD estimates $\hatGx, \hatGy$ of the true Gramians $\Gx(\theta), \Gy(\theta)$, respectively. Then the gradient of $g(\theta)$ (equation~\eqref{eq:gravity_sum}) can be approximated by the gradient (w.r.t. $\theta$) of
\begin{equation}
\label{eq:gravity_estimate}
\hatg_i(\theta, \hatGx, \hatGy) \coloneqq \braket{\ux_i(\theta)}{\hatGy \ux_i(\theta)} +
\braket{\uy_i(\theta)}{\hatGx \uy_i(\theta)},
\end{equation}
as stated in the following proposition.
\begin{proposition}
\label{prop:unbiased}
If $i$ is drawn uniformly from $\{1, \dots, n\}$, and $\hatGx, \hatGy$ are unbiased estimates of $\Gx(\theta), \Gy(\theta)$ and independent of $i$, then $\nabla_\theta \hatg_i(\theta, \hatGx, \hatGy)$ is an unbiased estimate of $\nabla g(\theta)$.
\end{proposition}
In a mini-batch setting, these estimates can be further averaged over a batch of examples $i \in B$ (which we do in our experiments), but we will omit batches to keep the notation concise.
Next, we propose several methods for maintaining the Gramian estimates $\hatGx, \hatGy$, and discuss their tradeoffs.
%%%%%%%%%%%%%%%%%%%%%%%%%%%%%%%%%%%%%%%%%%%%%%%%%%%%%%%%%%%%%%%%%%%%%%%%%%%%%%%
\begin{algorithm}[tb]
   \caption{{SAGram (Stochastic Average Gramian)}}
   \label{alg:sagram}
\begin{algorithmic}[1]
    \STATE {\bfseries Input:} Training data $\{(\xf_i, \yf_i, s_i)\}_{i \in \{1, \dots, n\}}$, learning rate $\eta > 0$.
    \STATE {\bfseries Initialization phase}
    \begin{ALC@g}
        \STATE draw $\theta$ randomly
        \STATE $\hatux_i \leftarrow \ux_i(\theta), \ \hatuy_i \leftarrow \uy_i(\theta) \quad \forall i \in \{1, \dots, n\}$
        \STATE ${\hat S_\ux} \leftarrow \frac{1}{n}\sum_{i = 1}^n \hatux_i \otimes \hatux_i$, \ ${\hat S_\uy} \leftarrow \frac{1}{n}\sum_{i = 1}^n \hatuy_i \otimes \hatuy_i$
    \end{ALC@g}
    \REPEAT
        \STATE Update Gramian estimates ($i \sim \text{Uniform}(n)$)
        \begin{ALC@g}
            \STATE $\hatGx \leftarrow {\hat S_\ux} + \beta [\ux_i(\theta) \otimes \ux_i(\theta) - \hatux_i \otimes \hatux_i]$, \quad
            $\hatGy \leftarrow {\hat S_\uy} + \beta [\uy_i(\theta) \otimes \uy_i(\theta) - \hatuy_i \otimes \hatuy_i]$
        \end{ALC@g}
        \STATE Update model parameters then update caches ($i \sim \text{Uniform}(n)$)
        \begin{ALC@g}
            \STATE $\theta \leftarrow \theta - \eta \nabla_\theta [f_i(\theta) + \lambda \hatg_i(\theta, \hatGx, \hatGy)]$% \hfill cf. eq.~\eqref{eq:gravity_estimate}
            \STATE ${\hat S_\ux} \leftarrow {\hat S_\ux} + \frac{1}{n} [\ux_i(\theta)\otimes \ux_i(\theta) - \hatux_i \otimes \hatux_i]$, \quad ${\hat S_\uy} \leftarrow {\hat S_\uy} + \frac{1}{n} [\uy_i(\theta)\otimes \uy_i(\theta) - \hatuy_i \otimes \hatuy_i]$
            \STATE $\hatux_i \leftarrow \ux_i(\theta), \ \hatuy_i \leftarrow \uy_i(\theta)$
        \end{ALC@g}
    \UNTIL{stopping criterion}
\end{algorithmic}
\end{algorithm}
%%%%%%%%%%%%%%%%%%%%%%%%%%%%%%%%%%%%%%%%%%%%%%%%%%%%%%%%%%%%%%%%%%%%%%%%%%%%%%%
\subsection{Stochastic Average Gramian}
\label{sec:estimation-sagram}
Inspired by variance reduction for Monte Carlo integrals~\citep{hammersley1964monte,evans2000approximating}, many variance reduction methods have been developed for stochastic optimization. In particular, stochastic average gradient methods~\citep{schmidt2017SAG,defazio2014saga} work by maintaining a cache of individual gradients, and estimating the full gradient using this cache.
Since each Gramian is a sum of outer-products (see equation~\eqref{eq:gramian_definition}), we can apply the same technique to estimate Gramians. For all $i \in \{1, \dots, n\}$, let $\hatux_i, \hatuy_i$ be a cache of the left and right embeddings respectively. We will denote by a superscript $(t)$ the value of a variable at iteration $t$. Let
${\hat S_\ux}^{(t)} = \frac{1}{n}\sum_{i = 1}^n \hatux^{(t)}_i \otimes \hatux^{(t)}_i$,
which corresponds to the Gramian based on the current caches. %The algorithm is initialized by drawing a random $\theta^{(0)}$ and making a pass over the training set to initialize the caches $\hat \ux_i^{(0)}, \hat \uy_j^{(0)}, \hat S_\ux^{(0)}, \hat S_\uy^{(0)}$. Then,
At each iteration $t$, an example $i$ is drawn uniformly at random and the estimate of the Gramian is given by
\begin{equation}
\label{eq:sagram_estimate}
\hatGx^{(t)} = {\hat S_\ux}^{(t)} + \beta [\ux_i(\theta^{(t)}) \otimes \ux_i(\theta^{(t)}) - \hatux^{(t)}_i \otimes \hatux^{(t)}_i],
\end{equation}
and similarly for $ \hatGy^{(t)}$. This is summarized in Algorithm~\ref{alg:sagram}, where the model parameters are updated using SGD (line 10), but can be replaced with any first-order method. Note that for efficient implementation, the sums ${\hat S_\ux}, {\hat S_\uy}$ are not recomputed at each step, they are updated in an online fashion (line 11). Here $\beta$ can take one of the following values:
\begin{enumerate}
\item $\beta = \frac{1}{n}$, following SAG~\citep{schmidt2017SAG}, or
\item $\beta = 1$, following SAGA~\citep{defazio2014saga}. 
\end{enumerate}
The choice of $\beta$ comes with trade-offs that we briefly discuss below. We will denote the cone of positive semi-definite $k \times k$ matrices by $\Scal^k_+$.

\begin{proposition}
Suppose $\beta = \frac{1}{n}$ in~\eqref{eq:sagram_estimate}. Then for all $t$, $\hatGx^{(t)}, \hatGy^{(t)}$ remain in $\Scal^k_+$.
\end{proposition}

\begin{proposition}
Suppose $\beta = 1$ in~\eqref{eq:sagram_estimate}. Then for all $t$, $\hatGx^{(t)}$ is an unbiased estimate of $\Gx(\theta^{(t)})$.
\end{proposition}
\vspace{-.01in}
While taking $\beta = 1$ gives an unbiased estimate, note that it does not guarantee that the estimates remain in $\Scal^k_+$. In practice, this can cause numerical issues, but can be avoided by projecting the estimates~\eqref{eq:sagram_estimate} on $\Scal^k_+$, using the eigenvalue decomposition of each estimate. The per-iteration computational cost of maintaining the Gramian estimates is $\Ocal(k)$ to update the caches, $\Ocal(k^2)$ to update the estimates ${\hat S_\ux}, {\hat S_\uy}, \hatGx, \hatGy$, and $\Ocal(k^3)$ for projecting on $\Scal_+^k$. Given the small size of $k$, $\Ocal(k^3)$ remains tractable. The memory cost is $\Ocal(nk)$, since each embedding needs to be cached (plus a negligible $\Ocal(k^2)$ for storing the Gramian estimates). Note that this makes SAGram much less expensive than applying the original SAG(A) methods, which require maintaining caches of the \emph{gradients}, which would incur a $\Ocal(nd)$ memory cost, where $d$ is the number of parameters of the model, and can be several orders of magnitude larger than the embedding dimension $k$. However, $\Ocal(nk)$ can still be prohibitively expensive when $n$ is very large. In the next section, we propose a different method which does not incur this additional memory cost, and does not require projection.

%%%%%%%%%%%%%%%%%%%%%%%%%%%%%%%%%%%%%%%%%%%%%%%%%%%%%%%%%%%%%%%%%%%%%%%%%%%%%%%
\subsection{Stochastic Online Gramian}
\label{sec:estimation-online}
To derive the second method, we reformulate problem~\eqref{eq:objective} as a two-player game. The first player optimizes over the parameters of the model $\theta$, the second player optimizes over the Gramian estimates $\hatGx, \hatGy \in \Scal^k_+$, and they seek to minimize the respective losses
\begin{equation}
\label{eq:game}
\begin{cases}
L_1^{\hatGx, \hatGy}(\theta) = \frac{1}{n}\sum_{i = 1}^n [f_i(\theta) + \lambda \hat{g}_i(\theta, \hatGx, \hatGy)] \\
L_2^{\theta}(\hatGx, \hatGy) = \frac{1}{2}\|\hatGx - \Gx(\theta)\|_F^2 + \frac{1}{2}\|\hatGy - \Gy(\theta)\|_F^2,
\end{cases}
\end{equation}
where $\hat g_i$ is defined in~\eqref{eq:gravity_estimate}, and $\|\cdot\|_F^{}$ denotes the Frobenius norm. To simplify the discussion, we will assume in this section that $f_i$ is differentiable. This reformulation can then be justified by characterizing its first-order stationary points, as follows.

\begin{proposition}
$(\theta, \hatGx, \hatGy)\in \Rbb^d \times \Scal^k_+ \times \Scal^k_+$ is a first-order stationary point for~\eqref{eq:game} if and only if $\theta$ is a first-order stationary point for problem~\eqref{eq:objective} and $\hatGx = \Gx(\theta), \hatGy = \Gy(\theta)$.
\end{proposition}

%Other methods have been proposed for multi-agent optimization on $\Scal^k_+$, for example, \cite{mertikopoulos2017distributed} apply stochastic matrix exponential learning and give local convergence results for locally stable equilibria, and these methods can similarly be applied for Gramian estimation.

Several stochastic first-order dynamics can be applied to the problem, and Algorithm~\ref{alg:sogram} gives a simple instance where each player implements SGD with constant learning rates, $\eta$ for player $1$ and $\alpha$ for player 2. In this case, the updates of the Gramian estimates (line 7) have a particularly simple form, since $\nabla_{\hatGx} L_2^\theta(\hatGx, \hatGy) = \hatGx - \Gx(\theta)$, which can be estimated by $\hatGx - \ux_i(\theta) \otimes \ux_i(\theta)$, resulting in the update
\begin{equation}
\label{eq:sogram_update}
\hat G_u^{(t)} = (1-\alpha)\hat G_u^{(t-1)} + \alpha \ux_i(\theta^{(t)}) \otimes \ux_i(\theta^{(t)}),
\end{equation}
and similarly for $\hat G_v$. One advantage of this form is that each update performs a convex combination between the current estimate and a rank-1 PSD matrix, thus guaranteeing that the estimates remain in $\Scal_+^k$, without the need to project. The per-iteration cost of updating the estimates is $\Ocal(k^2)$, and the memory cost is $\Ocal(k^2)$ for storing the Gramians, which are both negligible.

The update~\eqref{eq:sogram_update} can also be interpreted as computing an online estimate of the Gramian by averaging rank-1 terms with decaying weights, thus we call the method Stochastic Online Gramian. Indeed, we have by induction on $t$,
\[
\hat G_u^{(t)} = \sum_{\tau = 1}^{t} \alpha(1-\alpha)^{t-\tau} u_{i_\tau}(\theta^{(\tau)}) \otimes u_{i_\tau}(\theta^{(\tau)}).
\]
Intuitively, the averaging reduces the variance of the estimator but introduces a bias, and the choice of the hyper-parameter $\alpha \in (0, 1)$ trades-off bias and variance. Similar smoothing of estimators has been observed to empirically improve convergence in other contexts, e.g.~\citep{mandt2014smoothed}. We give coarse estimates of this tradeoff under mild assumptions in the next proposition.
\begin{proposition}
\label{prop:bias-variance}
Let $\bar G^{(t)}_u = \sum_{\tau = 1}^t \alpha(1-\alpha)^{t-\tau}G_u(\theta^{(\tau)})$. Suppose that there exist $ \sigma, \delta > 0$ such that for all~$t$, $\Exp_{i \sim \text{Uniform}} \|u_i(\theta^{(t)}) \otimes u_i(\theta^{(t)}) - G_u(\theta^{(t)})\|_F^2 \leq \sigma^2$ and $\|G_\ux(\theta^{(t+1)}) - G_\ux(\theta^{(t)})\|_F \leq \delta$. Then $\forall t$,
\vspace{-.08in}
\begin{align}
\label{eq:variance_bound}
\Exp \|\hat G^{(t)}_u - \bar G^{(t)}_u\|_F^2 &\leq \sigma^2 \frac{\alpha}{2 - \alpha}\\
\label{eq:bias_bound}
\|\bar G_\ux^{(t)} - G_\ux^{(t)}\|_F &\leq \delta (1/\alpha - 1) + (1-\alpha)^t \|G_\ux^{(t)}\|_F.
\end{align}
\end{proposition}
The first assumption simply bounds the variance of single-point estimates, while the second bounds the distance between two consecutive Gramians (a reasonable assumption, since in practice the changes in Gramians vanish as the trajectory $\theta^{(\tau)}$ converges). In the limiting case $\alpha = 1$, $\hat G_u^{(t)}$ reduces to a single-point estimate, in which case the bias~\eqref{eq:bias_bound} vanishes and the variance~\eqref{eq:variance_bound} is maximal, while smaller values of $\alpha$ decrease variance and increase bias. This is confirmed in our experiments, as discussed in Section~\ref{sec:experiments}.

%%%%%%%%%%%%%%%%%%%%%%%%%%%%%%%%%%%%%%%%%%%%%%%%%%%%%%%%%%%%%%%%%%%%%%%%%%%%%%%
\begin{algorithm}[tb]
   \caption{SOGram (Stochastic Online Gramian)}
   \label{alg:sogram}
\begin{algorithmic}[1]
    \STATE {\bfseries Input:} Training data $\{(\xf_i, \yf_i, s_i)\}_{i \in \{1, \dots, n\}}$, learning rates $\eta > 0$, $\alpha \in (0, 1)$.
    \STATE {\bfseries Initialization phase}
    \begin{ALC@g}
        \STATE draw $\theta$ randomly
        \STATE $\hatGx, \hatGy \leftarrow 0^{k\times k}$
    \end{ALC@g}
    \REPEAT
        \STATE Update Gramian estimates ($i \sim \text{Uniform}(n)$)
        \begin{ALC@g}
            \STATE $\hatGx \leftarrow (1 - \alpha)\hatGx + \alpha \ux_i(\theta)\otimes \ux_i(\theta)$, \quad $\hatGy \leftarrow (1 - \alpha)\hatGy + \alpha \uy_i(\theta)\otimes \uy_i(\theta)$
        \end{ALC@g}
        \STATE Update model parameters ($i \sim \text{Uniform}(n)$)
        \begin{ALC@g}
            \STATE $\theta \leftarrow \theta - \eta \nabla_\theta [f_i(\theta) + \lambda \hatg_i(\theta, \hatGx, \hatGy)]$ %\hfill cf. eq.~\eqref{eq:gravity_estimate}
        \end{ALC@g}
    \UNTIL{stopping criterion}
\end{algorithmic}
\end{algorithm}
\vspace{-.1in}
%%%%%%%%%%%%%%%%%%%%%%%%%%%%%%%%%%%%%%%%%%%%%%%%%%%%%%%%%%%%%%%%%%%%%%%%%%%%%%%

\subsection{Comparison with sampling methods}
\label{sec:discussion}
We conclude this section by observing that traditional sampling methods can be recast in terms of the Gramian formulation~\eqref{eq:gravity_dot}, and implementing them in this form can decrease their computional complexity in the large batch regime. Indeed, suppose a batch $B \subset \{1, \dots, n\}$ is sampled, and the gravity term $g(\theta)$ is approximated by
\begin{equation}
\label{eq:sampling_naive}
\tilde g(\theta) = \frac{1}{|B|^2} \sum_{i \in B} \sum_{j \in B} \braket{u_i(\theta)}{v_j(\theta)}^2.
\end{equation}
Then applying a similar transformation to Section~\ref{sec:gram_formulation}, one can show that
{\begin{equation}
\label{eq:sampling_gram}
\tilde g(\theta) = \Big\langle\frac{1}{|B|}\sum_{i \in B} \rankone{u_i(\theta)}, \frac{1}{|B|}\sum_{j \in B} \rankone{v_j(\theta)}\Big\rangle.
\end{equation}}%
The double-sum formulation~\eqref{eq:sampling_naive} involves a sum of $|B|^2$ inner products of vectors in $\Rbb^k$, thus computing its gradient costs $\Ocal(k |B|^2)$. The Gramian formulation~\eqref{eq:sampling_gram}, on the other hand, is the inner product of two $k \times k$ matrices, each involving a sum of $|B|$ terms, thus computing the gradient in this form costs $\Ocal(k^2|B|)$, which can give significant computational savings when $|B|$ is larger than the embedding dimension $k$, a common situation in practice. Incidentally, given expression~\eqref{eq:sampling_gram}, sampling methods can be interpreted as implicitly computing Gramian estimates, using a sum of rank-1 terms over the batch. Intuitively, one advantage of SOGram and SAGram is that they take into account many more embeddings (by caching or online averaging) than is possible using plain sampling.

%%%%%%%%%%%%%%%%%%%%%%%%%%%%%%%%%%%%%%%%%%%%%%%%%%%%%%%%%%%%%%%%%%%%%%%%%%%%%%%
%%%%%%%%%%%%%%%%%%%%%%%%%%%%%%%%%%%%%%%%%%%%%%%%%%%%%%%%%%%%%%%%%%%%%%%%%%%%%%%
%%%%%%%%%%%%%%%%%%%%%%%%%%%%%%%%%%%%%%%%%%%%%%%%%%%%%%%%%%%%%%%%%%%%%%%%%%%%%%%
\section{Experiments}
\label{sec:experiments}
In this section, we conduct large-scale experiments on the Wikipedia dataset~\citep{wikipedia}. Additional experiments on the MovieLens dataset~\citep{harper2015movielens} are given in Appendix~\ref{app:movielens}.

\subsection{Experimental setup}
\label{sec:exp-setup}
{\bf Datasets} We consider the problem of learning the intra-site links between Wikipedia pages. Given a pair of pages $(x, y) \in \Xcal \times \Xcal$, the target similarity is $1$ if there is a link from $x$ to $y$, and $0$ otherwise. Here a page is represented by a feature vector $x = (x_{\fid}, x_{\fngrams}, x_{\fcats})$, where $x_{\fid}$ is (a one-hot encoding of) the page URL, $x_{\fngrams}$ is a bag-of-words representation of the set of n-grams of the page's title, and $x_\fcats$ is a bag-of-words representation of the categories the page belongs to. Note that the left and right feature spaces coincide in this case, but the target similarity is not necessarily symmetric (the links are directed edges).
We carry out our experiments on subsets of the Wikipedia graph corresponding to three languages: Simple English, French, and English, denoted respectively by \texttt{simple}, \texttt{fr}, and \texttt{en}. These subgraphs vary in size, and Table~\ref{tbl:dataset} shows some basic statistics for each set. Each set is partitioned into training and validation using a (90\%, 10\%) split.

\begin{table}[h]
\centering
{\small\begin{tabular}{l|r|r|r|r}
\hline
language & \# pages & \# links & \# ngrams & \# cats \\
 \hline
 \texttt{simple} & 85K & 4.6M &  8.3K & 6.1K\\
 \hline
 \texttt{fr} & 1.8M &  142M & 167.4K & 125.3K \\
 \hline
 \texttt{en} & 5.3M & 490M & 501.0K & 403.4K\\
 \hline
 \end{tabular}}%
  \vspace{.05in}
 \caption{Corpus sizes for each training set.}\label{tbl:dataset}
 \vspace{-.1in}
\end{table}

{\bf Model} We train a non-linear embedding model consisting of a two-tower neural network as in Figure~\ref{fig:model}, where the left and right embedding functions map, respectively, the source and destination page features. Both networks have the same structure: the input feature embeddings are concatenated then mapped through two hidden layers with ReLU activations. The input feature embeddings are shared between the two networks, and their dimensions are $50$ for \texttt{simple}, $100$ for \texttt{fr}, and $120$ for \texttt{en}. The sizes of the hidden layers are $[256, 64]$ for \texttt{simple} and $[512, 128]$ for \texttt{fr} and \texttt{en}.

{\bf Training} The model is trained using SAGram, SOGram, and batch negative sampling as a baseline. We use a learning rate $\eta = 0.01$ and a gravity coefficient $\lambda = 10$ (cross-validated). All of the methods use a batch size $1024$. For SAGram and SOGram, a batch $B$ is used in the Gramian updates (line 8 in Algorithm~\ref{alg:sagram} and line 7 in Algorithm~\ref{alg:sogram}, where we use a sum of rank-1 terms over the batch), and another batch $B'$ is used in the gradient computation\footnote{We use two separate batches to ensure the independence assumption of Proposition~\ref{prop:unbiased}}. For the sampling method, the gravity term is approximated by all cross-pairs $(i, j) \in B \times B'$, and for efficiency, we implement it using the Gramian formulation as discussed in Section~\ref{sec:discussion}, since we operate in a regime where the batch size is an order of magnitude larger than the embedding dimension $k$ (equal to $64$ for \texttt{simple} and $128$ for \texttt{fr} and \texttt{en}).

\begin{figure}[h]
\centering
\includegraphics[width=\textwidth]{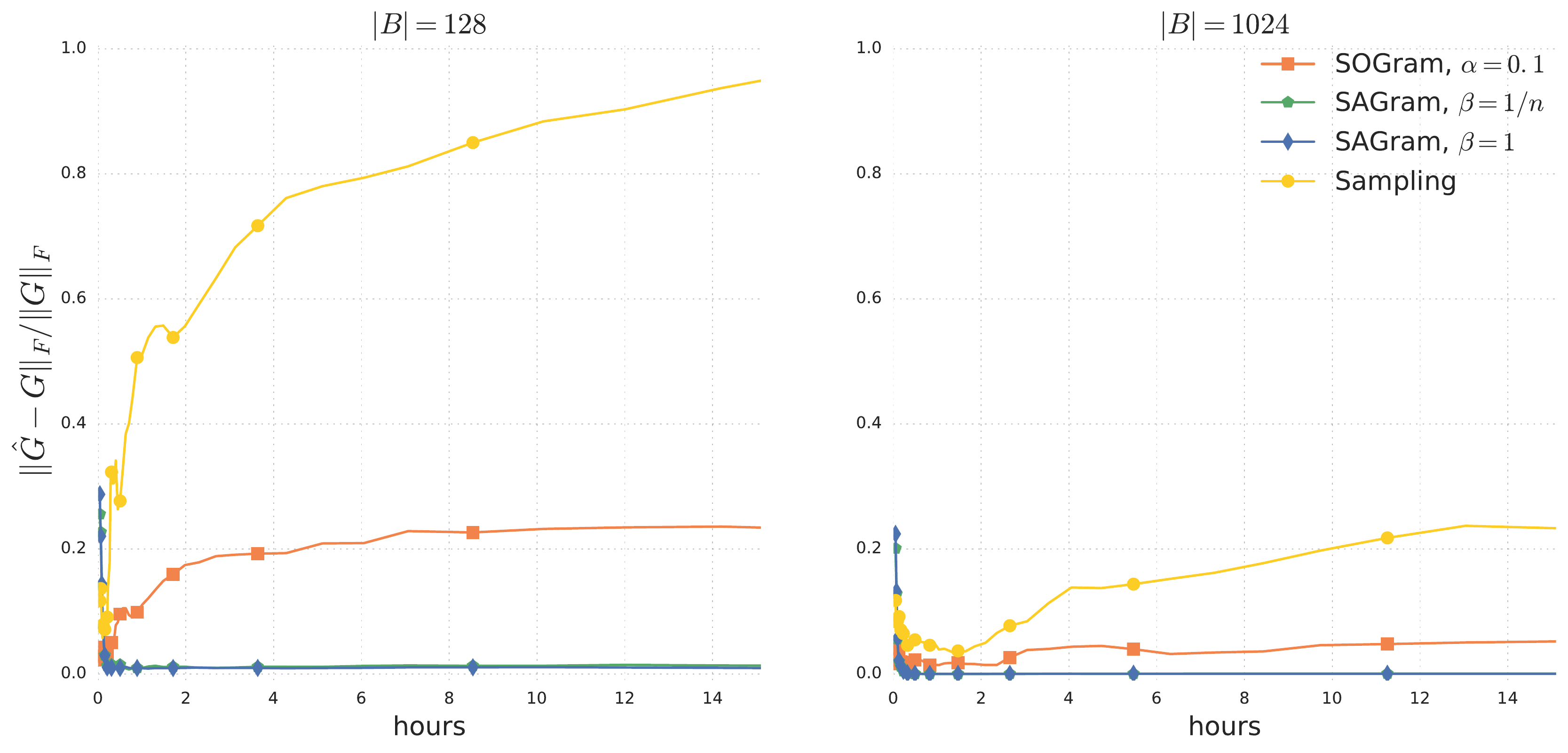}%
   \caption{Gramian estimation error on a common trajectory $(\theta^{(t)})$, for SAGram, SOGram and sampling.}
\label{fig:gram_estimation}
\end{figure}

\subsection{Quality of Gramian estimates}
In the first set of experiments, we evaluate the quality of the Gramian estimates using each method. In order to have a meaningful comparison, we fix a trajectory of model parameters $(\theta^{(t)})_{t \in \{1, \dots, T\}}$, and evaluate how well each method tracks the true Gramians $\Gx(\theta^{(t)}), \Gy(\theta^{(t)})$ on that common trajectory. This experiment is done on \texttt{simple}, the smallest of the datasets, so that we can compute the exact Gramians by periodically computing the embeddings $u_i(\theta^{(t)}), v_i(\theta^{(t)})$ on the full training set at a given time~$t$. We report the estimation error for each method, measured by the normalized Frobenius distance $\frac{\|\hatGx^{(t)} - \Gx(\theta^{(t)})\|_F}{\|\Gx(\theta^{(t)})\|_F}$ in Figure~\ref{fig:gram_estimation}. We can observe that both variants of SAGram yield the best estimates, and that SOGram yields better estimates than sampling. We also vary the batch size to evaluate its impact: increasing the batch size from 128 to 1024 improves the quality of all estimates, as expected. It is worth noting that the estimates of SOGram with $|B| = 128$ have comparable quality to sampling estimates with $|B| = 1024$.

In Figure~\ref{fig:gram_tradeoff}, we evaluate the bias-variance tradeoff discussed in Section~\ref{sec:estimation-online}, by comparing the estimates of SOGram with different learning rates~$\alpha$. We observe that for the initial iterations, higher values of $\alpha$ yield better estimates, but as training progresses, the errors decay to a lower value for lower $\alpha$ (observe in particular how all the plots intersect). This is consistent with the results of Proposition~\ref{prop:bias-variance}: higher values of~$\alpha$ induce higher variance which persists throughout training, while a lower value of $\alpha$ reduces the variance but introduces a bias, which is mostly visible during the early iterations, but decreases as the trajectory converges. We further study the SOGram estimates on the larger datasets in Appendix~\ref{app:gramian_experiments}.

\begin{figure}[h]
\centering
\includegraphics[width=.51\textwidth]{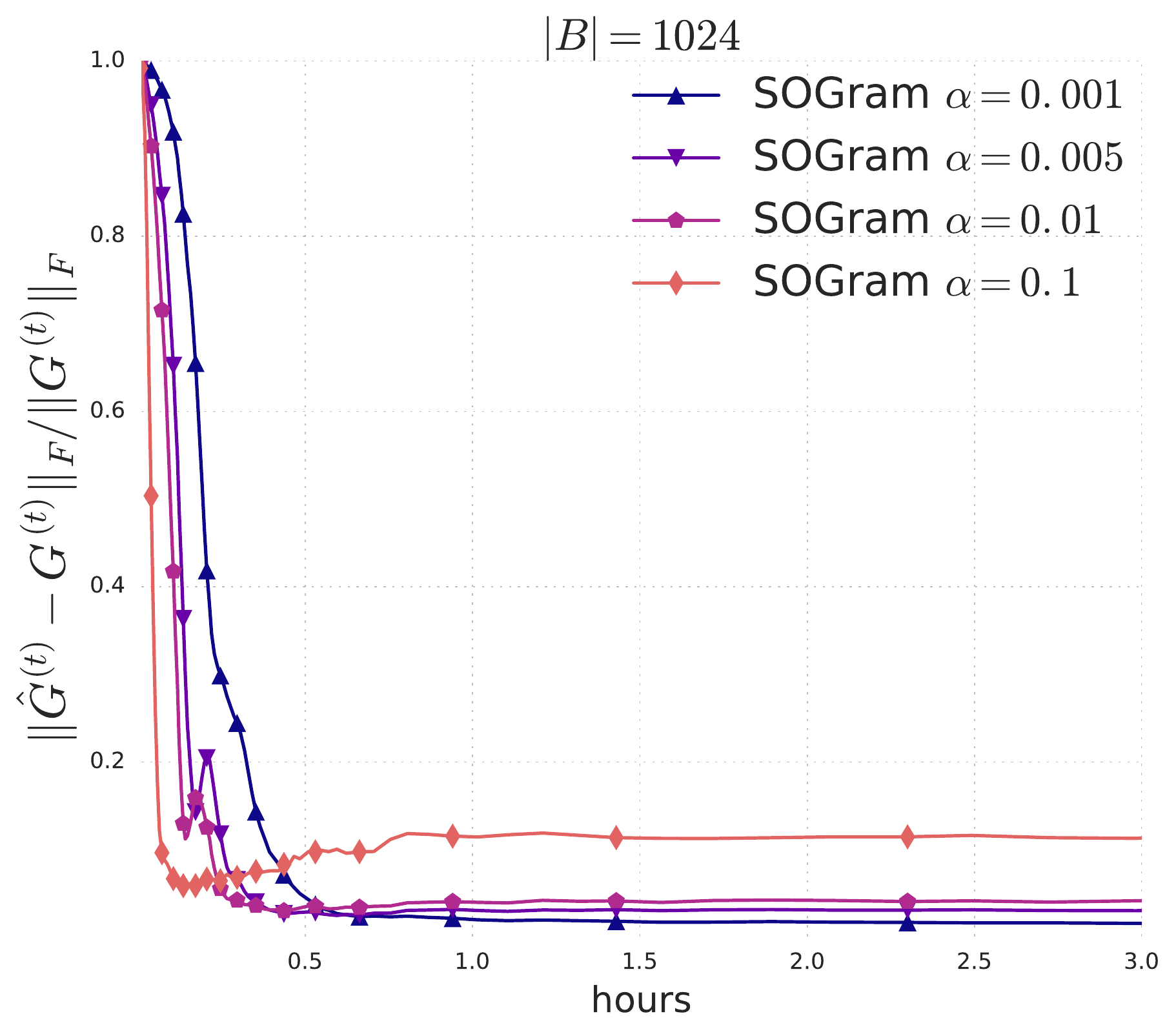}%
   \caption{Gramian estimation error of SOGram, for different values of $\alpha$.}
\label{fig:gram_tradeoff}
\end{figure}

\subsection{Impact on training speed and generalization quality}
In order to evaluate the impact of the Gramian estimation quality on training speed and generalization quality, we compare the validation performance of batch sampling and SOGram with different Gramian learning rates $\alpha$, on each dataset (we do not use SAGram due to its prohibitive memory cost for corpus sizes of 1M or more). We estimate the mean average precision (MAP) at 10, by periodically (every 5 minutes) scoring left items in the validation set against 50K random candidates -- exhuastively scoring all candidates is prohibitively expensive at this scale, but this gives a reasonable approximation.
\begin{figure}[h!]
\centering
\def\wdt{.49\textwidth}
\includegraphics[width=\wdt]{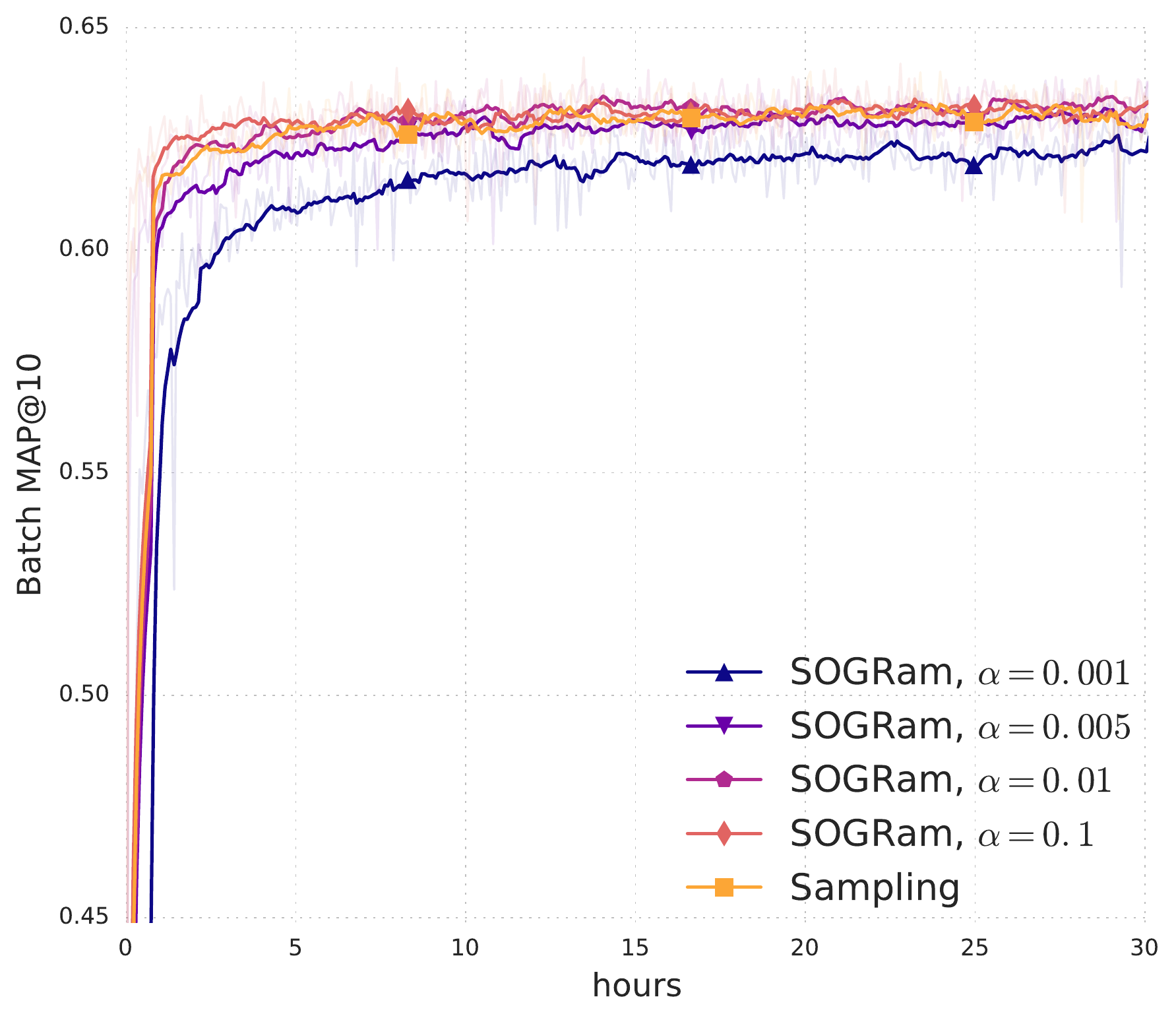}
\includegraphics[width=\wdt]{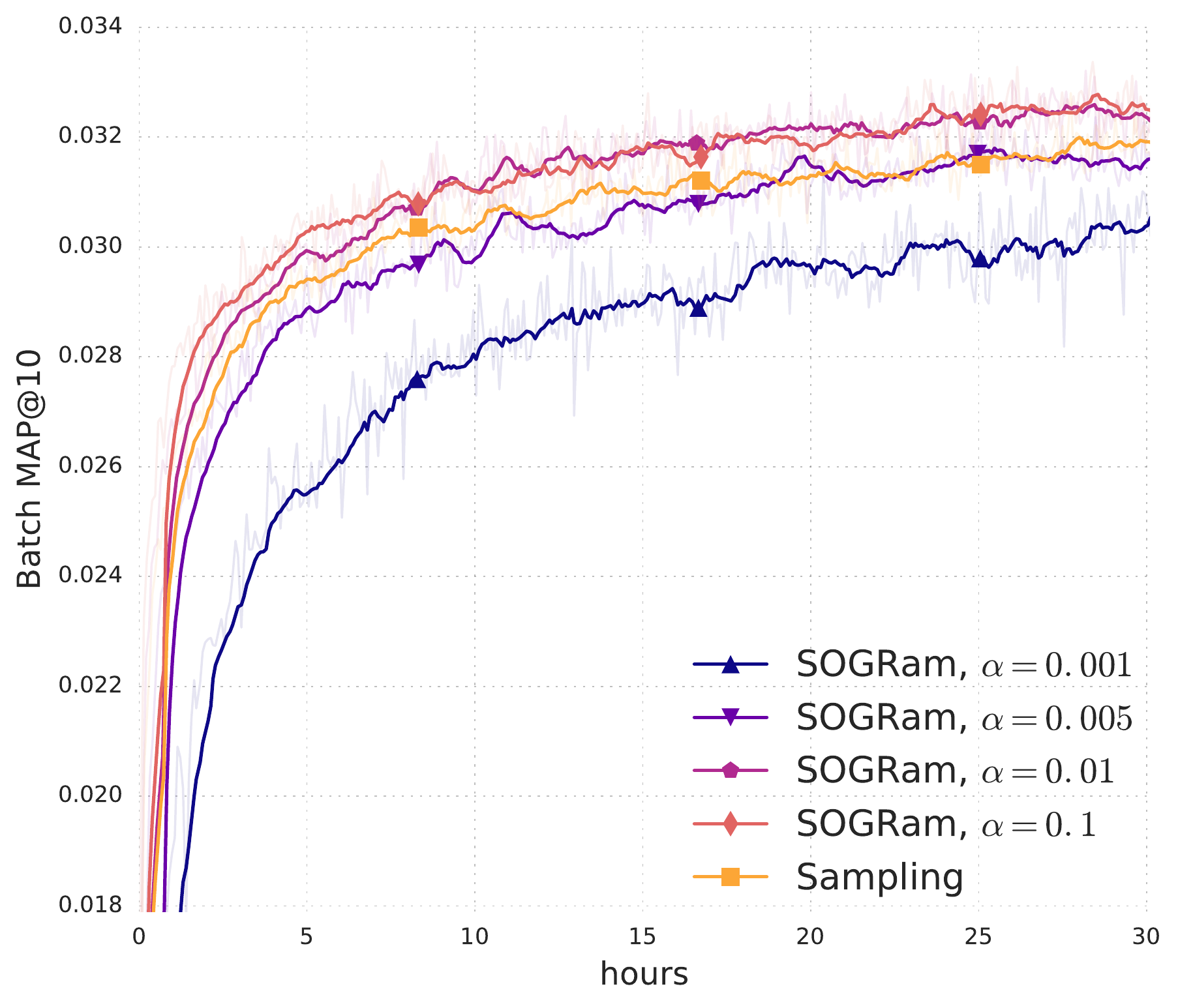}
\includegraphics[width=\wdt]{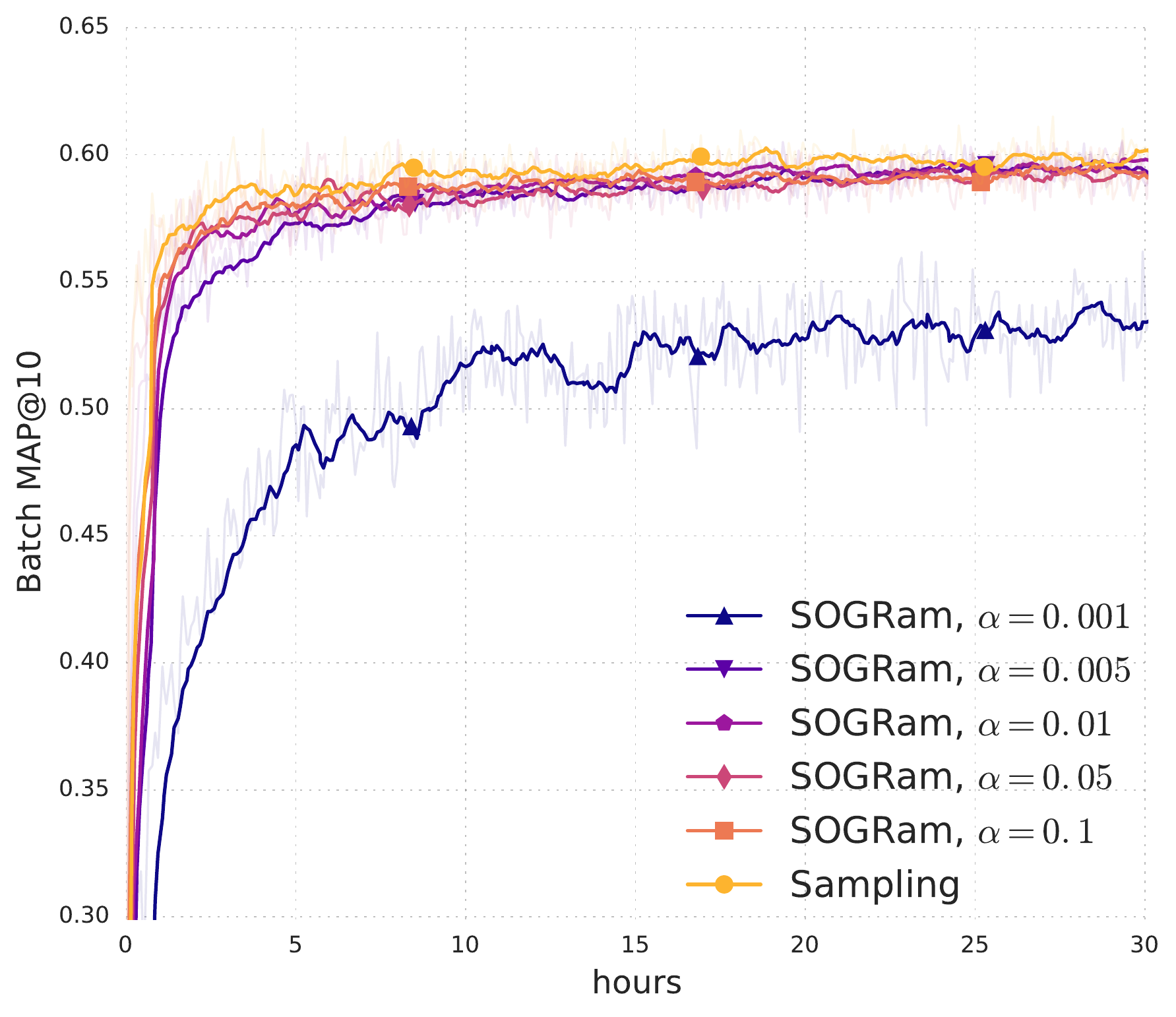}
\includegraphics[width=\wdt]{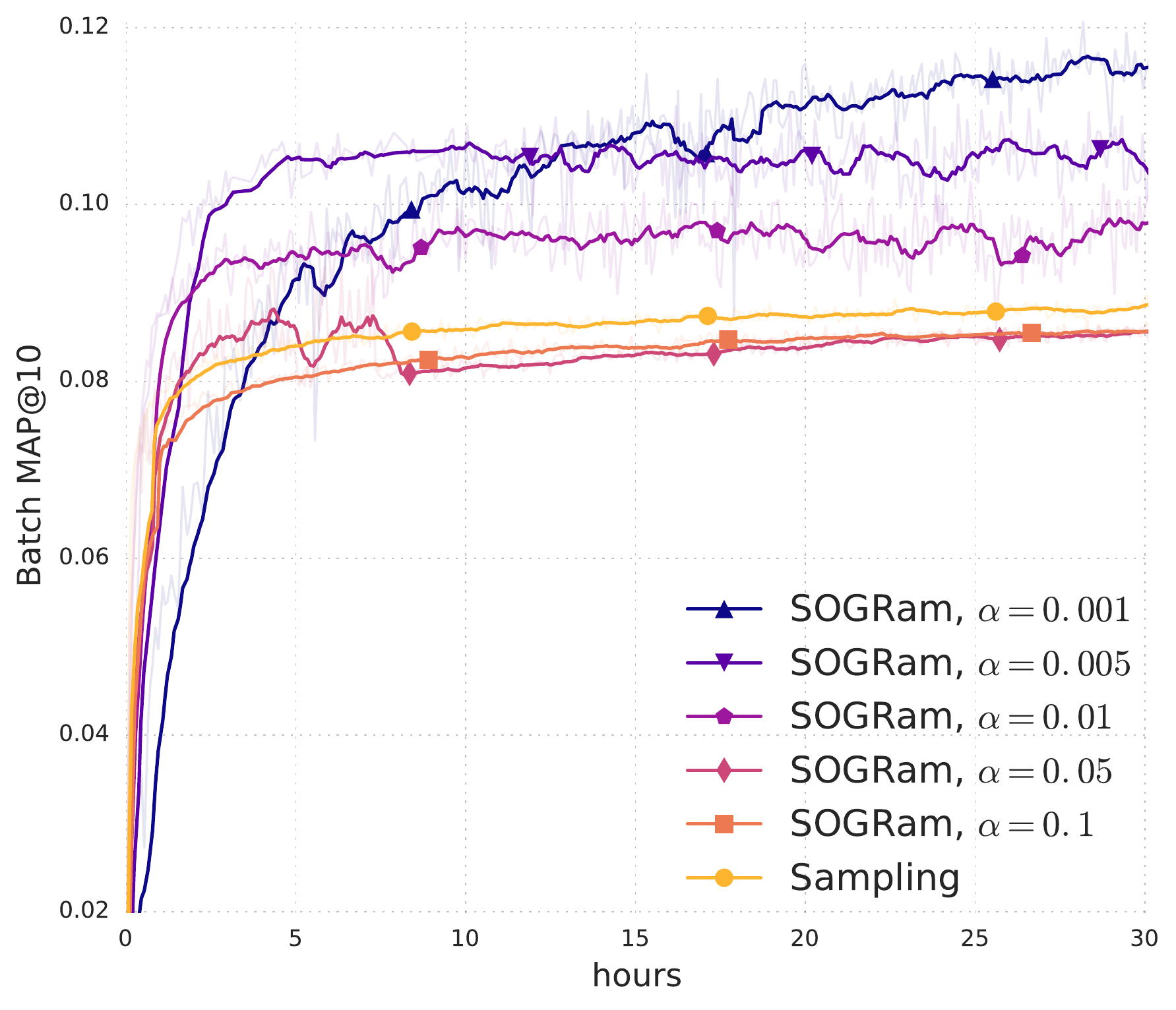}
\includegraphics[width=\wdt]{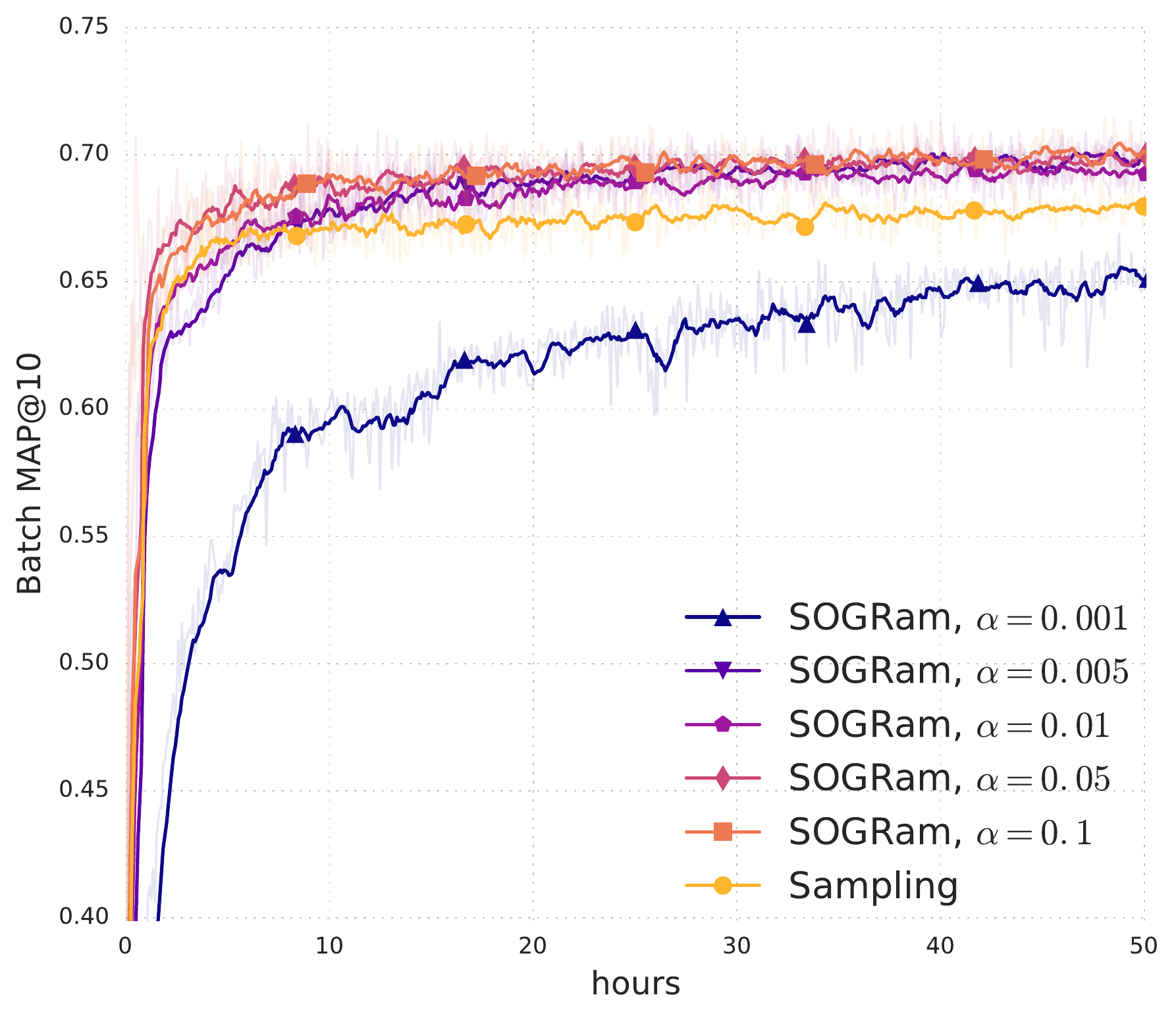}
\includegraphics[width=\wdt]{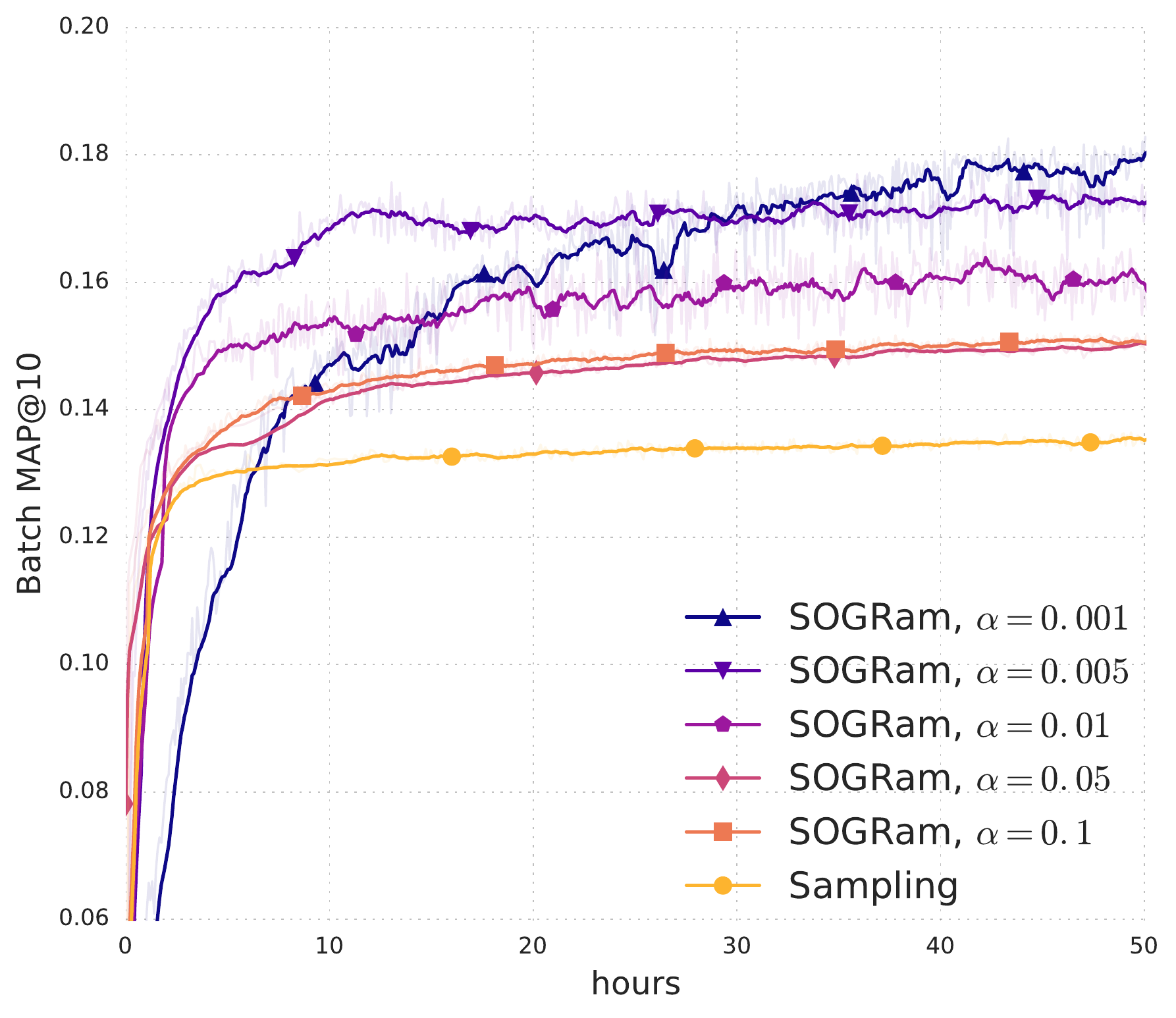}
\caption{Mean average precision at 10 on the training set (left), and the validation set (right), for different methods, on \texttt{simple} (top), \texttt{fr} (middle), and \texttt{en} (bottom).}
\vspace{-.1in}
\label{fig:validation-MAP}
\end{figure}

The results are reported in Figure~\ref{fig:validation-MAP}. While SOGram does not improve the MAP on the training set compared to the baseline sampling method, it consistently achieves the best validation performance, by a large margin for the larger sets. This discrepancy between training and validation can be explained by the fact that the gravity term $g(\theta)$ has a regularizing effect, and by better estimating this term, SOGram improves generalization. Table~\ref{tbl:validation} summarizes the relative improvement of the final validation MAP.

\begin{table}[h]
\centering
{\footnotesize\begin{tabular}{l|c|c|c|c|c}
\hline
language & Sampling & SOGram (0.001) & SOGram (0.005) &  SOGram (0.01) & SOGram (0.1)\\
 \hline
 \texttt{simple} & 0.0319 & 0.0306 (-4.0\%) & 0.0317 (-0.6\%) & \bf 0.0325 (+1.8\%) & 0.0324 (+1.5\%) \\
 \hline
 \texttt{fr} & 0.0886 & \bf 0.1158 (+30.7 \%) & 0.1049 (+18.4 \%) & 0.0983 (+10.9 \%) & 0.0857 (-3.3 \%)\\
 \hline
 \texttt{en} & 0.1352 & \bf 0.1801 (+33.2 \%) & 0.1725 (+27.6 \%) & 0.1593 (+17.8 \%) & 0.1509 (+11.6 \%)\\
 \hline
 \end{tabular}}%
 \vspace{.05in}
 \caption{Final validation MAP on each dataset, and relative improvement compared to batch sampling.}\label{tbl:validation}
 \vspace{-.1in}
\end{table}

The improvement on \texttt{simple} is modest (1.8\%), which can be explained by the relatively small corpus size (85K unique pages), in which case the baseline sampling already yields decent estimates. On the larger corpora, we obtain a much more significant improvement of 30.7\% on \texttt{fr} and 33.2\% on \texttt{en}. The plots for \texttt{en} and \texttt{fr} also reflect the bias-variance tradeoff dicussed in Proposition~\ref{prop:bias-variance}: with a lower $\alpha$, progress is initially slower (due to the bias introduced in the Gramian estimates), but the final performance is better. Given a limited training time budget, one may prefer a higher $\alpha$, and it is worth observing that with $\alpha = 0.01$ on \texttt{en}, SOGram achieves a better performance under 2 hours of training, than batch sampling in 50 hours. This tradeoff also motivates the use of decaying Gramian learning rates, which we leave for future experiments.

%%%%%%%%%%%%%%%%%%%%%%%%%%%%%%%%%%%%%%%%%%%%%%%%%%%%%%%%%%%%%%%%%%%%%%%%%%%%%%%
\section{Conclusion}
\label{sec:conclusion}
We showed that the Gramian formulation commonly used in low-rank matrix factorization can be leveraged for training non-linear embedding models, by maintaining estimates of the Gram matrices and using them to estimate the gradient. By applying variance reduction techniques to the Gramians, one can improve the quality of the gradient estimates, without relying on large sample size as is done in traditional sampling methods. This leads to a significant impact on training time and generalization quality, as indicated by our experiments. An important direction of future work is to extend this formulation to a larger family of penalty functions, such as the spherical loss family studied in~\citep{vincent2015efficient,brebisson2016exploration}.

%%%%%%%%%%%%%%%%%%%%%%%%%%%%%%%%%%%%%%%%%%%%%%%%%%%%%%%%%%%%%%%%%%%%%%%%%%%%%%%
%%%%%%%%%%%%%%%%%%%%%%%%%%%%%%%%%%%%%%%%%%%%%%%%%%%%%%%%%%%%%%%%%%%%%%%%%%%%%%%
\bibliographystyle{abbrvnat}
\bibliography{bib}

\begin{thebibliography}{30}
\providecommand{\natexlab}[1]{#1}
\providecommand{\url}[1]{\texttt{#1}}
\expandafter\ifx\csname urlstyle\endcsname\relax
  \providecommand{\doi}[1]{doi: #1}\else
  \providecommand{\doi}{doi: \begingroup \urlstyle{rm}\Url}\fi

\bibitem[Agarwal and Chen(2009)]{agarwal2009regression}
D.~Agarwal and B.-C. Chen.
\newblock Regression-based latent factor models.
\newblock In \emph{Proceedings of the 15th ACM SIGKDD International Conference
  on Knowledge Discovery and Data Mining}, KDD '09, pages 19--28, New York, NY,
  USA, 2009. ACM.

\bibitem[Bai et~al.(2017)Bai, Goldman, and Zhang]{bai2017tapas}
Y.~Bai, S.~Goldman, and L.~Zhang.
\newblock Tapas: Two-pass approximate adaptive sampling for softmax.
\newblock \emph{CoRR}, abs/1707.03073, 2017.

\bibitem[Bayer et~al.(2017)Bayer, He, Kanagal, and Rendle]{bayer2017generic}
I.~Bayer, X.~He, B.~Kanagal, and S.~Rendle.
\newblock A generic coordinate descent framework for learning from implicit
  feedback.
\newblock In \emph{Proceedings of the 26th International Conference on World
  Wide Web}, WWW '17, pages 1341--1350, 2017.

\bibitem[Bengio and Senecal(2003)]{bengio2003quick}
Y.~Bengio and J.~Senecal.
\newblock Quick training of probabilistic neural nets by importance sampling.
\newblock In \emph{Proceedings of the Ninth International Workshop on
  Artificial Intelligence and Statistics, {AISTATS} 2003, Key West, Florida,
  USA, January 3-6, 2003}, 2003.

\bibitem[Bengio and Senecal(2008)]{bengio2008adaptive}
Y.~Bengio and J.~Senecal.
\newblock Adaptive importance sampling to accelerate training of a neural
  probabilistic language model.
\newblock \emph{{IEEE} Trans. Neural Networks}, 19\penalty0 (4):\penalty0
  713--722, 2008.

\bibitem[Bromley et~al.(1993)Bromley, Bentz, Bottou, Guyon, LeCun, Moore,
  S{\"a}ckinger, and Shah]{bromley1993signature}
J.~Bromley, J.~W. Bentz, L.~Bottou, I.~Guyon, Y.~LeCun, C.~Moore,
  E.~S{\"a}ckinger, and R.~Shah.
\newblock Signature verification using a "siamese" time delay neural network.
\newblock \emph{International Journal of Pattern Recognition and Artificial
  Intelligence}, 7\penalty0 (4):\penalty0 669--688, 1993.

\bibitem[Chechik et~al.(2010)Chechik, Sharma, Shalit, and
  Bengio]{chechik2010large}
G.~Chechik, V.~Sharma, U.~Shalit, and S.~Bengio.
\newblock Large scale online learning of image similarity through ranking.
\newblock \emph{J. Mach. Learn. Res.}, 11:\penalty0 1109--1135, Mar. 2010.

\bibitem[Chen et~al.(2016)Chen, Grangier, and Auli]{chen2016strategies}
W.~Chen, D.~Grangier, and M.~Auli.
\newblock Strategies for training large vocabulary neural language models.
\newblock In \emph{Proceedings of the 54th Annual Meeting of the Association
  for Computational Linguistics, {ACL} 2016}, 2016.

\bibitem[de~Br\'{e}bisson and Vincent(2016)]{brebisson2016exploration}
A.~de~Br\'{e}bisson and P.~Vincent.
\newblock An exploration of softmax alternatives belonging to the spherical
  loss family.
\newblock \emph{CoRR}, abs/1511.05042, 2016.

\bibitem[Defazio et~al.(2014)Defazio, Bach, and
  Lacoste-Julien]{defazio2014saga}
A.~Defazio, F.~Bach, and S.~Lacoste-Julien.
\newblock Saga: A fast incremental gradient method with support for
  non-strongly convex composite objectives.
\newblock In Z.~Ghahramani, M.~Welling, C.~Cortes, N.~D. Lawrence, and K.~Q.
  Weinberger, editors, \emph{Advances in Neural Information Processing Systems
  27}, pages 1646--1654. Curran Associates, Inc., 2014.

\bibitem[Evans and Swartz(2000)]{evans2000approximating}
M.~Evans and T.~Swartz.
\newblock \emph{Approximating Integrals via Monte Carlo and Deterministic
  Methods}.
\newblock Oxford Statistical Science Series. Oxford University Press, Oxford,
  2000.

\bibitem[Grover and Leskovec(2016)]{grover2016node2vec}
A.~Grover and J.~Leskovec.
\newblock Node2vec: Scalable feature learning for networks.
\newblock In \emph{Proceedings of the 22Nd ACM SIGKDD International Conference
  on Knowledge Discovery and Data Mining}, KDD '16, pages 855--864, New York,
  NY, USA, 2016. ACM.
\newblock ISBN 978-1-4503-4232-2.

\bibitem[Hammersley and Handscomb(1964)]{hammersley1964monte}
J.~Hammersley and D.~Handscomb.
\newblock \emph{Monte Carlo Methods}.
\newblock Monographs on Applied Probability and Statistics Series. John Wiley
  \& Sons, Incorporated, 1964.

\bibitem[Harper and Konstan(2015)]{harper2015movielens}
F.~M. Harper and J.~A. Konstan.
\newblock The movielens datasets: History and context.
\newblock \emph{ACM Transactions on Interactive Intelligent Systems}, 2015.

\bibitem[Hu et~al.(2008)Hu, Koren, and Volinsky]{hu2008collborative}
Y.~Hu, Y.~Koren, and C.~Volinsky.
\newblock Collaborative filtering for implicit feedback datasets.
\newblock In \emph{Proceedings of the 2008 Eighth IEEE International Conference
  on Data Mining}, ICDM '08, pages 263--272, 2008.

\bibitem[Levy and Goldberg(2014)]{levy2014neural}
O.~Levy and Y.~Goldberg.
\newblock Neural word embedding as implicit matrix factorization.
\newblock In Z.~Ghahramani, M.~Welling, C.~Cortes, N.~D. Lawrence, and K.~Q.
  Weinberger, editors, \emph{Advances in Neural Information Processing Systems
  27}, pages 2177--2185. Curran Associates, Inc., 2014.

\bibitem[Mandt and Blei(2014)]{mandt2014smoothed}
S.~Mandt and D.~Blei.
\newblock Smoothed gradients for stochastic variational inference.
\newblock In Z.~Ghahramani, M.~Welling, C.~Cortes, N.~D. Lawrence, and K.~Q.
  Weinberger, editors, \emph{Advances in Neural Information Processing Systems
  27}, pages 2438--2446. Curran Associates, Inc., 2014.

\bibitem[Mikolov et~al.(2013)Mikolov, Chen, Corrado, and
  Dean]{mikolov2013word2vec}
T.~Mikolov, K.~Chen, G.~Corrado, and J.~Dean.
\newblock Efficient estimation of word representations in vector space.
\newblock \emph{CoRR}, abs/1301.3781, 2013.

\bibitem[Neyshabur and Srebro(2015)]{neyshabur2015symmetric}
B.~Neyshabur and N.~Srebro.
\newblock On symmetric and asymmetric lshs for inner product search.
\newblock In \emph{Proceedings of the 32Nd International Conference on
  International Conference on Machine Learning - Volume 37}, ICML'15, pages
  1926--1934. JMLR.org, 2015.

\bibitem[Pennington et~al.(2014)Pennington, Socher, and
  Manning]{pennington2014glove}
J.~Pennington, R.~Socher, and C.~D. Manning.
\newblock Glove: Global vectors for word representation.
\newblock In \emph{Empirical Methods in Natural Language Processing (EMNLP)},
  pages 1532--1543, 2014.

\bibitem[Qiu et~al.(2018)Qiu, Dong, Ma, Li, Wang, and Tang]{qiu2018network}
J.~Qiu, Y.~Dong, H.~Ma, J.~Li, K.~Wang, and J.~Tang.
\newblock Network embedding as matrix factorization: Unifying deepwalk, line,
  pte, and node2vec.
\newblock In \emph{Proceedings of the Eleventh ACM International Conference on
  Web Search and Data Mining}, WSDM '18, pages 459--467, New York, NY, USA,
  2018. ACM.
\newblock ISBN 978-1-4503-5581-0.

\bibitem[Rendle(2010)]{rendle2010FM}
S.~Rendle.
\newblock Factorization machines.
\newblock In \emph{Proceedings of the 2010 IEEE International Conference on
  Data Mining}, ICDM '10, pages 995--1000, Washington, DC, USA, 2010. IEEE
  Computer Society.

\bibitem[Schmidt et~al.(2017)Schmidt, Le~Roux, and Bach]{schmidt2017SAG}
M.~Schmidt, N.~Le~Roux, and F.~Bach.
\newblock Minimizing finite sums with the stochastic average gradient.
\newblock \emph{Math. Program.}, 162\penalty0 (1-2):\penalty0 83--112, Mar.
  2017.

\bibitem[Schroff et~al.(2015)Schroff, Kalenichenko, and
  Philbin]{schroff2015facenet}
F.~Schroff, D.~Kalenichenko, and J.~Philbin.
\newblock Facenet: A unified embedding for face recognition and clustering.
\newblock In \emph{2015 IEEE Conference on Computer Vision and Pattern
  Recognition (CVPR)}, pages 815--823, June 2015.

\bibitem[Shazeer et~al.(2016)Shazeer, Doherty, Evans, and
  Waterson]{shazeer2016swivel}
N.~Shazeer, R.~Doherty, C.~Evans, and C.~Waterson.
\newblock Swivel: Improving embeddings by noticing what's missing.
\newblock \emph{CoRR}, abs/1602.02215, 2016.

\bibitem[Shrivastava and Li(2014)]{shrivastava2014ALSH}
A.~Shrivastava and P.~Li.
\newblock Asymmetric lsh (alsh) for sublinear time maximum inner product search
  (mips).
\newblock In \emph{Proceedings of the 27th International Conference on Neural
  Information Processing Systems - Volume 2}, NIPS'14, pages 2321--2329,
  Cambridge, MA, USA, 2014. MIT Press.

\bibitem[Vincent et~al.(2015)Vincent, de~Br\'{e}bisson, and
  Bouthillier]{vincent2015efficient}
P.~Vincent, A.~de~Br\'{e}bisson, and X.~Bouthillier.
\newblock Efficient exact gradient update for training deep networks with very
  large sparse targets.
\newblock In C.~Cortes, N.~D. Lawrence, D.~D. Lee, M.~Sugiyama, and R.~Garnett,
  editors, \emph{Advances in Neural Information Processing Systems 28}, pages
  1108--1116. Curran Associates, Inc., 2015.

\bibitem[{Wikimedia Foundation}()]{wikipedia}
{Wikimedia Foundation}.
\newblock Wikimedia downloads.
\newblock \url{https://dumps.wikimedia.org/}.

\bibitem[Xin et~al.(2017)Xin, Mayoraz, Pham, Lakshmanan, and
  Anderson]{xin2017folding}
D.~Xin, N.~Mayoraz, H.~Pham, K.~Lakshmanan, and J.~R. Anderson.
\newblock Folding: Why good models sometimes make spurious recommendations.
\newblock In \emph{Proceedings of the Eleventh ACM Conference on Recommender
  Systems}, RecSys '17, pages 201--209, New York, NY, USA, 2017. ACM.

\bibitem[Yu et~al.(2017)Yu, Bilenko, and Lin]{yu2017selection}
H.-F. Yu, M.~Bilenko, and C.-J. Lin.
\newblock Selection of negative samples for one-class matrix factorization.
\newblock In \emph{Proceedings of the 2017 SIAM International Conference on
  Data Mining}, pages 363--371, 2017.

\end{thebibliography}
%%%%%%%%%%%%%%%%%%%%%%%%%%%%%%%%%%%%%%%%%%%%%%%%%%%%%%%%%%%%%%%%%%%%%%%%%%%%%%%
%%%%%%%%%%%%%%%%%%%%%%%%%%%%%%%%%%%%%%%%%%%%%%%%%%%%%%%%%%%%%%%%%%%%%%%%%%%%%%%

\newpage
\appendix
\section{Proofs}
\setcounter{proposition}{0}
\begin{proposition}
If $i$ is drawn uniformly in $\{1, \dots, n\}$, and $\hatGx, \hatGy$ are unbiased estimates of $\Gx(\theta), \Gy(\theta)$ and independent of $i$, then $\nabla_\theta \hatg_i(\theta, \hatGx, \hatGy)$ is an unbiased estimate of $\nabla g(\theta)$.
\end{proposition}
\begin{proof}
Starting from the expression~\eqref{eq:gravity} of $g(\theta) = \braket{G_u(\theta)}{G_v(\theta)}$, and applying the chain rule, we have
\begin{align}
\nabla g(\theta) 
&= \nabla \braket{\Gx(\theta)}{\Gy(\theta)} \notag \\
&= J_\ux(\theta)[\Gy(\theta)] + J_\uy(\theta)[\Gx(\theta)], \label{eq:gravity_gradient}
\end{align}
where $J_\ux(\theta)$ denotes the Jacobian of $\Gx(\theta)$, an order-three tensor given by
\begin{align*}
J_\ux(\theta)_{l, i, j} = \frac{\partial \Gx(\theta)_{i, j}}{\partial \theta_l}, && l \in \{1, \dots, d\}, i, j \in \{1, \dots, n\},
\end{align*} and $J_\ux(\theta)[\Gy(\theta)]$ denotes the vector $[\sum_{i, j} J_\ux(\theta)_{l, i, j} \Gy(\theta)_{i, j}]_{l \in \{1, \dots, d\}}$.

Observing that $\hatg_i(\theta, \hat G_u, \hat G_v) = \braket{\hat G_u}{u_i(\theta)\otimes u_i(\theta)} + \braket{\hat G_v}{v_i(\theta) \otimes v_i(\theta)}$, and applying the chain rule, we have
\begin{equation}
\label{eq:gravity_estimate_gradient}
\nabla_\theta \hatg_i(\theta, \hatGx, \hatGy) 
= J_{u, i}(\theta)[\hatGy] + J_{\uy, i}(\theta)[\hatGx],
\end{equation}
where $J_{\ux, i}(\theta)$ is the Jacobian of $\ux_i(\theta) \otimes \ux_i(\theta)$, and
\[
\Exp_{i \sim \text{Uniform}}[J_{\ux, i}(\theta)] = \frac{1}{n} \sum_{i = 1}^n J_{u, i}(\theta) = J_u(\theta),
\]
an similarly for $J_{\uy, i}$. We conclude by taking expectations in~\eqref{eq:gravity_estimate_gradient} and using assumption that $\hat G_u, \hat G_v$ are independent of $i$.
\end{proof}

\begin{proposition}
Suppose $\beta = \frac{1}{n}$ in~\eqref{eq:sagram_estimate}. Then for all $t$, $\hatGx^{(t)}, \hatGy^{(t)}$ remain in $\Scal^k_+$.
\end{proposition}
\begin{proof}
From~\eqref{eq:sagram_estimate} and the definition of ${\hat S_\ux}^{(t)}$, we have,
\[
\hatGx^{(t)} = \frac{1}{n} \sum_{j \neq i} \hatux^{(t)}_j\otimes \hatux^{(t)}_j + \frac{1}{n} \ux_i(\theta^{(t)}) \otimes \ux_i(\theta^{(t)}),
\]
which is a sum of matrices in the PSD cone $\Scal^k_+$.
\end{proof}

\begin{proposition}
Suppose $\beta = 1$ in~\eqref{eq:sagram_estimate}. Then for all $t$, $\hatGx^{(t)}$ is an unbiased estimate of $\Gx(\theta^{(t)})$, and similarly for $ \hatGy^{(t)}$.
\end{proposition}
\begin{proof}
Denoting by $(\Fcal_t)_{t \geq 0}$ the filtration generated by the sequence $(\theta^{(t)})_{t \geq 0}$, and taking conditional expectations in~\eqref{eq:sagram_estimate}, we have
\begin{align*}
\mathbb E[\hatGx^{(t)}|\Fcal_t] 
&= {\hat S_\ux}^{(t)} + \underset{i\sim \text{Uniform}}{\mathbb E} [\ux_i(\theta^{(t)}) \otimes \ux_i(\theta^{(t)}) - \hatux_i^{(t)} \otimes \hatux_i^{(t)} | \Fcal_t] \\
&= {\hat S_\ux}^{(t)} + \frac{1}{n} \sum_{i = 1}^n [\ux_i(\theta^{(t)}) \otimes \ux_i(\theta^{(t)}) - \hatux_i \otimes \hatux_i] \\
&= \frac{1}{n} \sum_{i = 1}^n \ux_i(\theta^{(t)}) \otimes \ux_i(\theta^{(t)}) = G_u(\theta^{(t)}).
\end{align*}
\end{proof}

\begin{proposition}
$(\theta, \hatGx, \hatGy)\in \Rbb^d \times \Scal^k_+ \times \Scal^k_+$ is a first-order stationary point for~\eqref{eq:game} if and only if $\theta$ is a first-order stationary point for problem~\eqref{eq:objective} and $\hatGx = \Gx(\theta), \hatGy = \Gy(\theta)$.
\end{proposition}

\begin{proof}
$(\theta, \hatGx, \hatGy) \in \Rbb^d \times \Scal^k_+ \times \Scal^k_+$ is a first-order stationary point  of the game if and only if
\begin{align}
&\nabla f(\theta) +\lambda (J_\ux(\theta)[\hatGy] + J_\uy(\theta)[\hatGx]) = 0 \label{eq:equiv_proof_1} \\
&\braket{\hatGx - \Gx(\theta)}{G' - \hatGx} \geq 0, \quad \forall G' \in \Scal^k_+ \label{eq:equiv_proof_2} \\
&\braket{\hatGy - \Gy(\theta)}{G' - \hatGy} \geq 0, \quad \forall G' \in \Scal^k_+ \label{eq:equiv_proof_3}
\end{align}
The second and third conditions simply states that $\nabla_{\hatGx} L_2^\theta(\hatGx, \hatGy)$ and $\nabla_{\hat G_v}L_2^\theta(\hat G_u, \hat G_v)$ define supporting hyperplanes of $\Scal^k_+$ at $\hatGx, \hatGy$, respectively.

Since $\Gx(\theta) \in \Scal^k_+$, condition \eqref{eq:equiv_proof_2} is equivalent to $\hatGx = \Gx(\theta)$ (and similarly, \eqref{eq:equiv_proof_3} is equivalent to $\hatGy = \Gy(\theta)$). Using the expression~\eqref{eq:gravity_gradient} of $\nabla g$, we get that (\ref{eq:equiv_proof_1}-\ref{eq:equiv_proof_3}) is equivalent to $\nabla f(\theta) + \lambda\nabla g(\theta) = 0$.
\end{proof}

%============================================================================================
\begin{proposition}
Let $\bar G^{(t)}_u = \sum_{\tau = 1}^t \alpha(1-\alpha)^{t-\tau} G_u(\theta^{(\tau)})$. Suppose that there exist $ \sigma, \delta > 0$ such that for all~$t$, $\Exp_{i \sim \text{Uniform(n)}} \|u_i(\theta^{(t)}) \otimes u_i(\theta^{(t)}) - G_u(\theta^{(t)})\|_F^2 \leq \sigma^2$ and $\|G_\ux(\theta^{(t+1)}) - G_\ux(\theta^{(t)})\|_F \leq \delta$. Then $\forall t$,
\begin{align}
\label{eq:app-variance_bound}
\Exp \|\hat G^{(t)}_u - \bar G^{(t)}_u\|_F^2 &\leq \sigma^2 \frac{\alpha}{2 - \alpha}\\
\label{eq:app-bias_bound}
\|\bar G_\ux^{(t)} - G_\ux^{(t)}\|_F &\leq \delta (1/\alpha - 1) + (1-\alpha)^t \|G_\ux^{(t)}\|_F.
\end{align}
\end{proposition}

\begin{proof}
We start by proving the first bound~\eqref{eq:app-variance_bound}. As stated in Section~\ref{sec:estimation-online}, we have, by induction on $t$, $\hat G_u^{(t)} = \sum_{\tau = 1}^t a_{t - \tau} u_{i_\tau}(\theta^{(t)}) \otimes u_{i_\tau}(\theta^{(t)})$, where $a_\tau = \alpha(1-\alpha)^\tau$. And by definition of $\bar G^{(t)}$, we have $\bar G_u^{(t)} = \sum_{\tau = 1}^t  a_{t -\tau} G_u(\theta^{(\tau)})$. Thus,
\begin{align*}
\hat G_\ux^{(t)} - \bar G_\ux^{(t)}
&= \sum_{\tau = 1}^t a_{t-\tau}\Delta_\ux^{(\tau)}
\end{align*}
where $\Delta_\ux^{(\tau)} = u_{i_\tau}(\theta^{(\tau)}) \otimes u_{i_\tau}(\theta^{(\tau)}) - G_\ux(\theta^{(\tau)})$ are zero-mean random variables. Thus, taking the second moment, and using the first assumption (which simply states that the variance of $\Delta_u^{(\tau)}$ is bounded by $\sigma^2$), we have
\begin{align*}
\Exp\|\hat G_u^{(t)} - \bar G_u^{(t)}\|_F^2
&= \Exp \left\|\sum_{\tau = 1}^t a_{t-\tau}\Delta_\ux^{(\tau)}\right\|_F^2 
= \sum_{\tau =1}^t a_{t-\tau}^2 \Exp \|\Delta_\ux^{(\tau)}\|_F^2 \\
&\leq \sigma^2 \alpha^2 \sum_{\tau = 0}^{t-1} (1-\alpha)^{2\tau}
= \sigma^2 \alpha^2 \frac{1 - (1-\alpha)^{2t}}{1 - (1-\alpha)^2} \\
&\leq \sigma^2 \frac{\alpha}{2 - \alpha},
\end{align*}
which proves the first inequality~\eqref{eq:app-variance_bound}.

To prove the second inequality, we start from the definition of $\bar G_u^{(t)}$:

%-------------------------------------------------------
\begin{align}
\|\bar G_\ux^{(t)} - G_\ux^{(t)}\|_F^{}
&= \| \sum_{\tau = 1}^t a_{t-\tau}(G_\ux^{(\tau)} - G_\ux^{(t)}) - (1-\alpha)^t G_\ux^{(t)}\|_F^{} \notag\\
&\leq \sum_{\tau = 1}^t a_{t-\tau} \|G_\ux^{(\tau)} - G_\ux^{(t)}\|_F + (1-\alpha)^t \|G_\ux^{(t)}\|_F^{}, \label{eq:bias_proof_1}
\end{align}
where the first equality uses that fact that $\sum_{\tau = 1}^t a_{t-\tau} = 1 - (1-\alpha)^t$. Focusing on the first term, and bounding $\|G_\ux^{(\tau)} - G_\ux^{(t)}\|_F \leq (t-\tau)\delta$ by the triangle inequality, we get
\begin{align}
\sum_{\tau = 1}^t a_{t-\tau} \|G_\ux^{(\tau)} - G_\ux^{(t)})\|_F
&\leq \delta \sum_{\tau = 1}^t a_{t-\tau}(t-\tau) 
= \delta \alpha \sum_{\tau = 0}^{t-1} \tau (1-\alpha)^\tau \notag\\
&= \delta \alpha(1-\alpha) \frac{d}{d\alpha} \left[-\sum_{\tau = 0}^{t-1} (1-\alpha)^\tau \right] \notag\\
& = \delta \alpha(1-\alpha) \frac{d}{d\alpha} \left[ -\frac{1-(1-\alpha)^t}{\alpha} \right] \notag\\
&\leq \delta \alpha(1-\alpha) \frac{1}{\alpha^2}. \label{eq:bias_proof_2}
\end{align}
Combining~\eqref{eq:bias_proof_1} and~\eqref{eq:bias_proof_2}, we get the desired inequality~\eqref{eq:app-bias_bound}.
\end{proof}

%Interpretation of Proposition~\ref{prop:bias-variance}

%\begin{align*}
%\Exp &\|\hat G_\ux^{(t)} - G_\ux^{(t)}\|_F^2 \\
%&= \Exp \left\|\sum_{\tau = 1}^t a_{t-\tau}\Delta_\ux^{(\tau)}\right\|_F^2 + 2\Exp \braket{\sum_{\tau = 1}^t a_{t-\tau}\Delta_\ux^{(\tau)}}{\bar G_\ux^{(t)} - G_\ux^{(t)}} + \Exp \|\bar G_\ux^{(t)} - G_\ux^{(t)}\|_F^2.
%\end{align*}
% The cross term is the dot-product of a martingale sequence with the bias terms, which is arguably negligible compared to the squared terms (as we observed in our numerical experiments).

%============================================================================================
%%%%%%%%%%%%%%%%%%%%%%%%%%%%%%%%%%%%%%%%%%%%%%%%%%%%%%%%%%%%%%%%%%%%%%%%%%%%%%%
%%%%%%%%%%%%%%%%%%%%%%%%%%%%%%%%%%%%%%%%%%%%%%%%%%%%%%%%%%%%%%%%%%%%%%%%%%%%%%
\section{Interpretation of the gravity term}
\label{app:gravity}
In this section, we briefly discuss different interpretations of the gravity term. Starting from the expression~\eqref{eq:gravity_dot} of $g(\theta)$ and the definition~\eqref{eq:gramian_definition} of the Gram matrices, we have
\begin{align}
g(\theta) = \braket{G_u(\theta)}{G_v(\theta)}
= \braket{\frac{1}{n} \sum_{i = 1}^n u_i(\theta) \otimes u_i(\theta)}{G_v(\theta)} = \frac{1}{n} \sum_{i = 1}^n \braket{u_i(\theta)}{G_v(\theta) u_i(\theta)},
\label{eq:gravity_quadratic}
\end{align}
which is a quadratic form in the left embeddings $u_i$ (and similarly for $v_j$, by symmetry). In particular, the partial derivative of the gravity term with respect to an embedding $u_i$ is
\begin{align*}
\frac{\partial g(\theta)}{\partial u_i} 
= \frac{2}{n} G_v(\theta) u_i(\theta)
= \frac{2}{n} \left[ \frac{1}{n}\sum_{j = 1}^n v_j(\theta)\otimes v_j(\theta)\right] u_i(\theta).
\end{align*}
Each term $(v_j \otimes v_j) u_i = v_j \braket{v_j}{u_i}$ is simply the projection of $u_i$ on $v_j$ (scaled by $\|v_j\|^2$). Thus the gradient of $g(\theta)$ with respect to $u_i$ is an average of scaled projections of $u_i$ on each of the right embeddings $v_j$, and moving in the direction of the negative gradient simply moves $u_i$ away from regions of the embedding space with a high density of left embeddings. This corresponds to the intuition discussed in the introduction: the purpose of the gravity term $g(\theta)$ is precisely to push left and right embeddings away from each other, to avoid placing embeddings of dissimilar items near each other, a phenomenon referred to as folding of the embedding space~\citep{xin2017folding}.

\begin{figure}[h!]
\centering
\begin{subfigure}[b]{.49\textwidth}
\includegraphics[width=\textwidth]{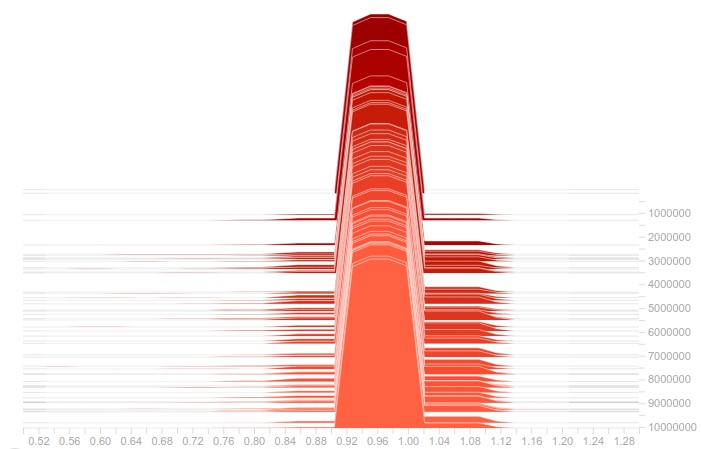}
\caption{$\lambda = 10^{-2}$, observed pairs.}
\end{subfigure}%
\begin{subfigure}[b]{.49\textwidth}
\includegraphics[width=\textwidth]{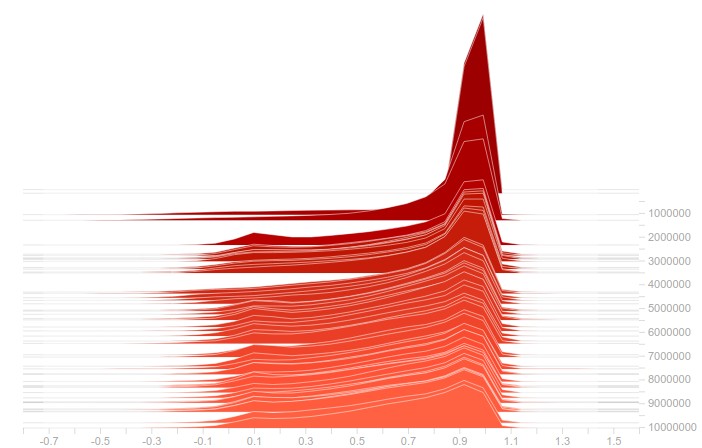}
\caption{$\lambda = 10^{-2}$, random pairs.}
\end{subfigure}
\begin{subfigure}[b]{.49\textwidth}
\includegraphics[width=\textwidth]{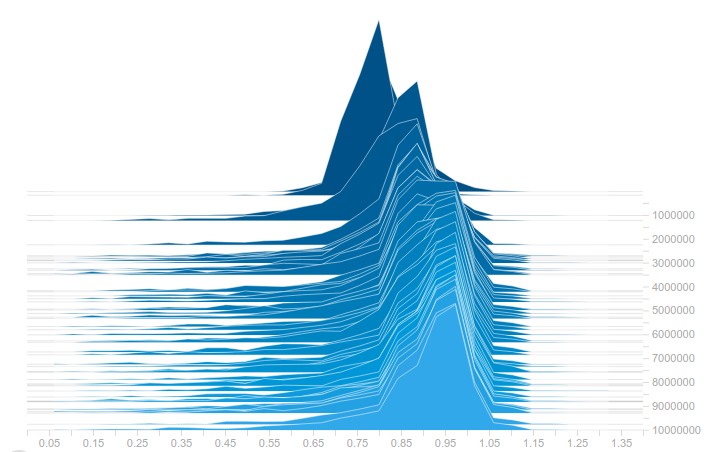}%
\caption{$\lambda = 10$, observed pairs.}
\end{subfigure}%
\begin{subfigure}[b]{.49\textwidth}
\includegraphics[width=\textwidth]{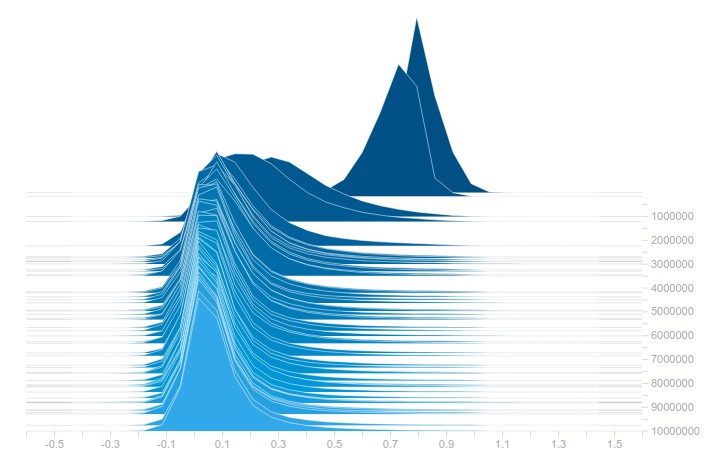}
\caption{$\lambda = 10$, random pairs.}
\end{subfigure}%
\caption{Evolution of the inner product distribution $\braket{u_i(\theta^{(t)})}{v_j(\theta^{(t)})}$ in the Wikipedia \texttt{en} model trained with different gravity coefficients  $\lambda$, for observed pairs (left) and random pairs (right).}
\label{fig:dot_distribution}
\end{figure}

\begin{figure}[h!]
\centering
\includegraphics[width=.5\textwidth]{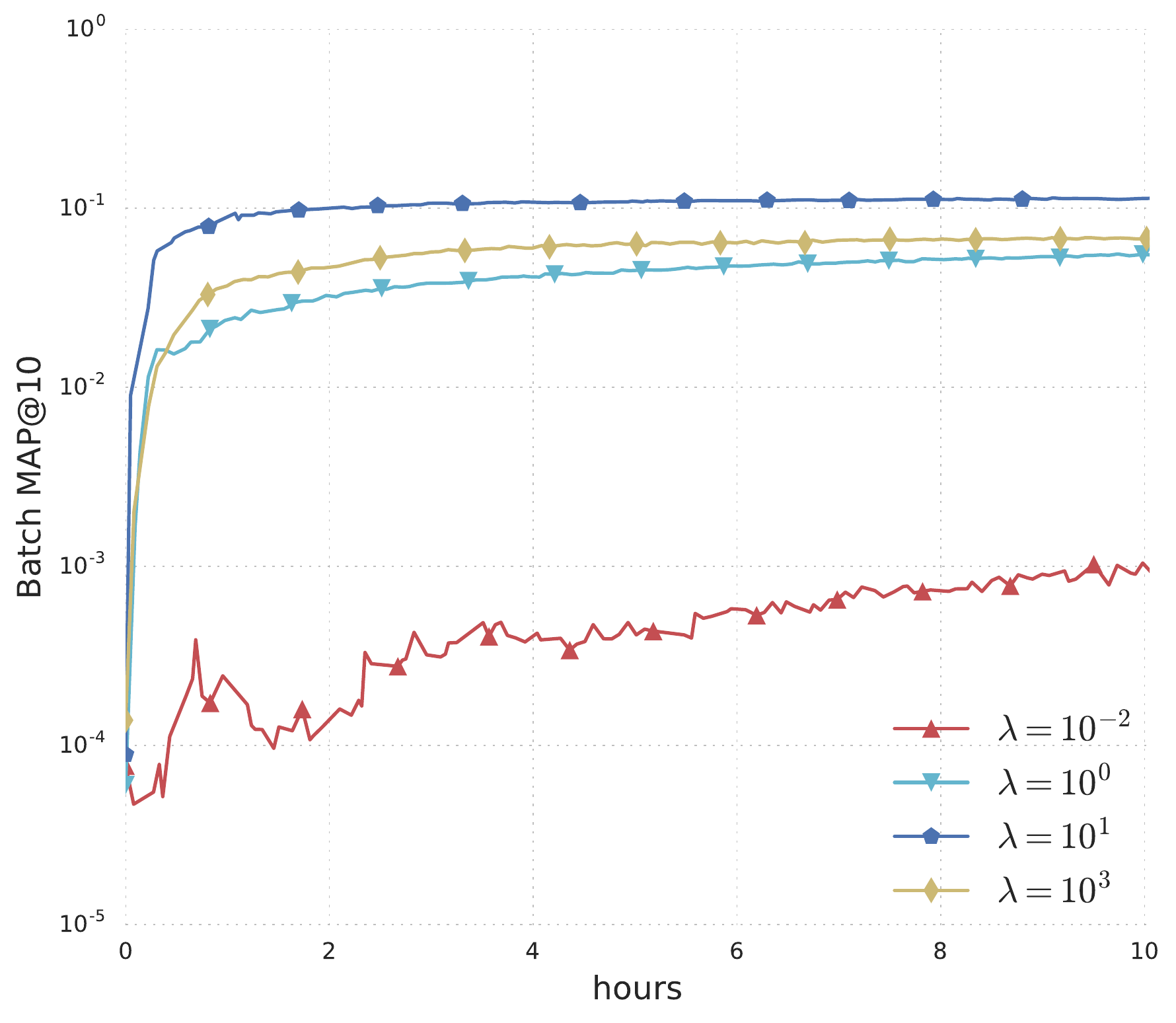}
\caption{Mean Average Precision of the Wikipedia \texttt{en} model, trained with different values of the gravity coefficient $\lambda$.}
\label{fig:effect_of_gravity_weight}
\end{figure}

In order to illustrate this effect of the gravity term on the embeddings, we visualize, in Figure~\ref{fig:dot_distribution}, the distribution of the inner product $\braket{u_i(\theta^{(t)})}{v_j(\theta^{(t)})}$, for random pairs $(i, j)$, and for observed pairs ($i = j$), and how these distributions change as $t$ increases. The plots are generated for the Wikipedia \texttt{en} model described in Section~\ref{sec:experiments}, trained with SOGram ($\alpha = 0.01$), with two different values of the gravity coefficient, $\lambda = 10^{-2}$ and $\lambda = 10$. In both cases, the distribution for observed pairs remains concentrated around values close to $1$, as one expects (recall that the target similarity is $1$ for observed pairs, i.e. pairs of connected pages in the Wikipedia graph). The distributions for random pairs, however, are very different: with $\lambda = 10$, the distribution quickly concentrates around a value close to $0$, while with $\lambda = 10^{-2}$ the distribution is more flat, and a large proportion of pairs have a high inner-product. This indicates that with a lower $\lambda$, the model is more likely to fold, i.e. place embeddings of unrelated items near each other. This is consistent with the validation MAP, reported in Figure~\ref{fig:effect_of_gravity_weight}. With $\lambda = 10^{-2}$, the validation MAP increases very slowly, and remains two orders of magnitude smaller than the model trained with $\lambda = 10$. The figure also shows that when the gravity coefficient is too large, the model is over-regularized and the MAP decreases.

To conclude this section, we also note that equation~\eqref{eq:gravity_quadratic} gives an intuitive motivation for the algorithms developed in this paper. Since the same quadratic form $\braket{\cdot}{G_v(\theta) \cdot}$ applies to all left embeddings $u_i$, maintaining an estimate $\hat G_v$ of $G_v(\theta)$ is much more efficient than estimating individual gradients (if one were to apply variance reduction to the gradients instead of the Gramians).

%============================================================================================
\section{Generalization to low-rank priors}
\label{app:low-rank}
So far, we have assumed a uniform zero prior to simplify the notation. In this section, we relax this assumption. Suppose that the prior is given by a low-rank matrix $P = Q R^\top$, where $Q, R \in \Rbb^{n \times k_P}$. In other words, the prior for a given pair $(i, j)$ is given by the dot product of two vectors $p_{ij} = \braket{q_i}{r_j}$. In practice, such a low-rank prior can be obtained, for example, by first training a simple low-rank matrix approximation of the similarity matrix $S$.

Given this low-rank prior, the penalty term~\eqref{eq:gravity_sum} becomes
\begin{align*}
g^P(\theta)
&= \frac{1}{n^2}\sum_{i = 1}^n \sum_{j = 1}^n [U_\theta V_\theta^\top - Q R^\top]_{ij}^2\\
&= \frac{1}{n^2}\braket{\Ux_\theta \Uy_\theta^\top - QR^\top}{\Ux_\theta \Uy_\theta^\top - QR^\top} \\
&= \frac{1}{n^2}\left[ \braket{\Ux_\theta^\top\Ux_\theta}{\Uy_\theta^\top \Uy_\theta} - 2\braket{\Ux_\theta^\top Q}{\Uy_\theta^\top R} + c \right],
\end{align*}%
where $c = \braket{Q^\top Q}{R^\top R}$ is a constant that does not depend on $\theta$. Here, we used a superscript $P$ in $g^P$ to disambiguate the zero-prior case.

Now, if we define weighted embedding matrices
\begin{align*}
\begin{cases}
H_\ux(\theta) \coloneqq \frac{1}{n}\Ux_\theta Q = \frac{1}{n}\sum_{i = 1}^n \ux_i(\theta)\otimes q_i \\
H_\uy(\theta) \coloneqq \frac{1}{n}\Uy_\theta R = \frac{1}{n}\sum_{i = 1}^n \uy_i(\theta)\otimes r_i,
\end{cases}
\end{align*}
the penalty term becomes
\[
g^P(\theta) = \braket{\Gx(\theta)}{\Gy(\theta)} - 2\braket{H_\ux(\theta)}{H_\uy(\theta)} + c.
\]
Finally, if we maintain estimates ${\hat H}_\ux, {\hat H}_\uy$ of $H_\ux(\theta), H_\uy(\theta)$, respectively (using the methods proposed in Section~\ref{sec:estimation}), we can approximate $\nabla g^P(\theta)$ by the gradient of
\begin{multline}
\label{eq:gravity_estimate_lr_prior}
\hatg^P_i(\theta, \hatGx, \hatGy, {\hat H}_\ux, {\hat H}_\uy) \coloneqq \\
\braket{\ux_i(\theta)}{\hatGy \ux_i(\theta)} + \braket{\uy_i(\theta)}{\hatGx \uy_i(\theta)} -  2\braket{\ux_i(\theta)}{{\hat H}_\uy q_i} - 2\braket{\uy_i(\theta)}{{\hat H}_\ux r_i}.
\end{multline}
Proposition~\ref{prop:unbiased} and Algorithms~\ref{alg:sagram} and~\ref{alg:sogram} can be generalized to the low-rank prior case by adding updates for ${\hat H}_\ux, {\hat H}_\uy$, and by using expression~\eqref{eq:gravity_estimate_lr_prior} of $\hatg^P_i$ when computing the gradient estimate.

\begin{proposition}
If $i$ is drawn uniformly in $\{1, \dots, n\}$, and $\hatGx$, $\hatGy$, ${\hat H}_\ux$, ${\hat H}_\uy$ are unbiased estimates of $\Gx(\theta)$, $\Gy(\theta)$, $H_\ux(\theta)$, $H_\uy(\theta)$, respectively, then $\nabla_\theta \hatg^P_i(\theta, \hatGx, \hatGy, {\hat H}_\ux, {\hat H}_\uy)$ is an unbiased estimate of $\nabla g^P(\theta)$.
\end{proposition}
\begin{proof}
Similar to the proof of Proposition~\ref{prop:unbiased}.
\end{proof}

The generalized versions of SOGram and SAGram are stated below, where we highlight the differences compared to the zero-prior versions. %Note that, unlike the Gramian matrices, the weighted embedding matrices $H_\ux, H_\uy$ are not symmetric, thus no projection is necessary when computing their estimates.

\newpage
\begin{algorithm}[h!]
   \caption{{SAGram (Stochastic Average Gramian) with low-rank prior}}
\begin{algorithmic}[1]
    \STATE {\bfseries Input:} Training data $\{(\xf_i, \yf_i, s_i)\}_{i \in \{1, \dots, n\}}$, \emp{low-rank priors $\{q_i, r_i\}_{i \in \{1, \dots, n\}}$}
    \STATE {\bfseries Initialization phase}
    \begin{ALC@g}
        \STATE draw $\theta$ randomly
        \STATE $\hatux_i \leftarrow \ux_i(\theta), \ \hatuy_i \leftarrow \uy_i(\theta) \quad \forall i \in \{1, \dots, n\}$
        \STATE ${\hat S_\ux} \leftarrow \frac{1}{n}\sum_{i = 1}^n \hatux_i \otimes \hatux_i$, ${\hat S_\uy} \leftarrow \frac{1}{n}\sum_{i = 1}^n \hatuy_i \otimes \hatuy_i$
        \emp{\STATE ${\hat T_\ux} \leftarrow \frac{1}{n}\sum_{i = 1}^n \hatux_i \otimes q_i$, ${\hat T_\uy} \leftarrow \frac{1}{n}\sum_{i = 1}^n \hatuy_i \otimes r_i$}
        %\STATE ${\hat T_\uy} \leftarrow \frac{1}{n}\sum_{i = 1}^n \hatuy_i \otimes r_i$}
    \end{ALC@g}
    \REPEAT
            \STATE Update Gramian estimates ($i \sim \text{Uniform}(n)$)
        \begin{ALC@g}
            \STATE $\hatGx \leftarrow {\hat S_\ux} + \beta [\ux_i(\theta) \otimes \ux_i(\theta) - \hatux_i \otimes \hatux_i]$, \quad
            $\hatGy \leftarrow {\hat S_\uy} + \beta [\uy_i(\theta) \otimes \uy_i(\theta) - \hatuy_i \otimes \hatuy_i]$
        \end{ALC@g}
        \emp{\STATE Update weighted embedding estimates\\
        \begin{ALC@g}
            \STATE $\hat H_\ux \leftarrow {\hat T_\ux} + \lambda [(\ux_i(\theta) - \hatux_i) \otimes q_i]$\\
            \STATE $\hat H_\uy \leftarrow {\hat T_\uy} + \lambda [(\uy_i(\theta) - \hatuy_i) \otimes r_i]$
        \end{ALC@g}
        }
        \STATE Update model parameters then update caches ($i \sim \text{Uniform}(n)$)
        \begin{ALC@g}
            \STATE $\theta \leftarrow \theta - \eta \nabla_\theta [f_i(\theta) + \lambda\emp{\hat g^P(\theta, \hatGx, \hatGy, \hat H_\ux, \hat H_\uy)}]$% \hfill cf. eq.~\eqref{eq:gravity_estimate}
            \STATE ${\hat S_\ux} \leftarrow {\hat S_\ux} + \frac{1}{n} [\ux_i(\theta)\otimes \ux_i(\theta) - \hatux_i \otimes \hatux_i]$, \quad ${\hat S_\uy} \leftarrow {\hat S_\uy} + \frac{1}{n} [\uy_i(\theta)\otimes \uy_i(\theta) - \hatuy_i \otimes \hatuy_i]$
            \emp{\STATE $\hat T_\ux \leftarrow {\hat T_\ux} + \frac{1}{n} [(\ux_i(\theta) - \hatux_i) \otimes q_i]$, \quad
             $\hat T_\uy \leftarrow {\hat T_\uy} + \frac{1}{n} [(\uy_i(\theta) - \hatuy_i) \otimes r_i]$
			}
            \STATE $\hatux_i \leftarrow \ux_i(\theta), \ \hatuy_i \leftarrow \uy_i(\theta)$
        \end{ALC@g}
    \UNTIL{stopping criterion}
\end{algorithmic}
\end{algorithm}

\begin{algorithm}[h!]
   \caption{SOGram (Stochastic Online Gramian) with low-rank prior}
\begin{algorithmic}[1]
    \STATE {\bfseries Input:} Training data $\{(\xf_i, \yf_i, s_i)\}_{i \in \{1, \dots, n\}}$, \emp{low-rank priors $\{q_i, r_i\}_{i \in \{1, \dots, n\}}$}
    \STATE {\bfseries Initialization phase}
    \begin{ALC@g}
        \STATE draw $\theta$ randomly
        \STATE $\hatGx, \hatGy \leftarrow 0^{k\times k}$
    \end{ALC@g}
    \REPEAT
        \STATE Update Gramian estimates ($i \sim \text{Uniform}(n)$)
        \begin{ALC@g}
            \STATE $\hatGx \leftarrow (1 - \alpha)\hatGx + \alpha \ux_i(\theta)\otimes \ux_i(\theta)$, \quad $\hatGy \leftarrow (1 - \alpha)\hatGy + \alpha \uy_i(\theta)\otimes \uy_i(\theta)$
        \end{ALC@g}
        \emp{\STATE Update weighted embedding estimates\\
        \begin{ALC@g}
            \STATE $\hat H_\ux \leftarrow (1-\alpha){\hat H_\ux} + \alpha \ux_i(\theta) \otimes q_i$, \quad $\hat H_\uy \leftarrow (1-\alpha)\hat H_\uy + \alpha \uy_i(\theta) \otimes r_i$
        \end{ALC@g}
        }
        \STATE Update model parameters ($i \sim \text{Uniform}(n)$)
        \begin{ALC@g}
            \STATE $\theta \leftarrow \theta - \eta \nabla_\theta [f_i(\theta) + \lambda\emp{\hat g^P(\theta, \hatGx, \hatGy, \hat H_\ux, \hat H_\uy)}]$
        \end{ALC@g}
    \UNTIL{stopping criterion}
\end{algorithmic}
\end{algorithm}

%============================================================================================
\newpage
\section{Further experiments on quality of Gramian estimates}
\label{app:gramian_experiments}
In addition to the experiment on Wikipedia \texttt{simple}, reported in Section~\ref{sec:experiments}, we also evaluated the quality of the Gramian esimates on Wikipedia \texttt{en}. Due to the large number of embeddings, computing the exact Gramians is no longer feasible, so we approximate it using a large sample of 1M embeddings. The results are reported in Figure~\ref{fig:gram_tradeoff_en}, which shows the normalized Frobenius distance between the Gramian estimates $\hat G_u$ and (the large sample approximation of) the true Gramian $G_u$. The results are similar to the experiment on \texttt{simple}: with a lower $\alpha$, the estimation error is initially high, but decays to a lower value as training progresses, which can be explained by the bias-variance tradeoff discussed in Proposition~\ref{prop:bias-variance}.

The tradeoff is affected by the trajectory of the true Gramians: smaller changes in the Gramians (captured by the parameter $\delta$ in Proposition~\ref{prop:bias-variance}) induce a smaller bias. In particular, changing the learning rate $\eta$ of the main algorithm can affect the performance of the Gramian estimates by affecting the rate of change of the true Gramians. To investiage this effect, we ran the same experiment with two different learning rates, $\eta = 0.01$ as in Section~\ref{sec:experiments}, and a lower learning rate $\eta = 0.002$. The errors converge to similar values in both cases, but the error decay occurs much faster with smaller $\eta$, which is consistent with our analysis.
\begin{figure}[h]
\centering
\includegraphics[width=.49\textwidth]{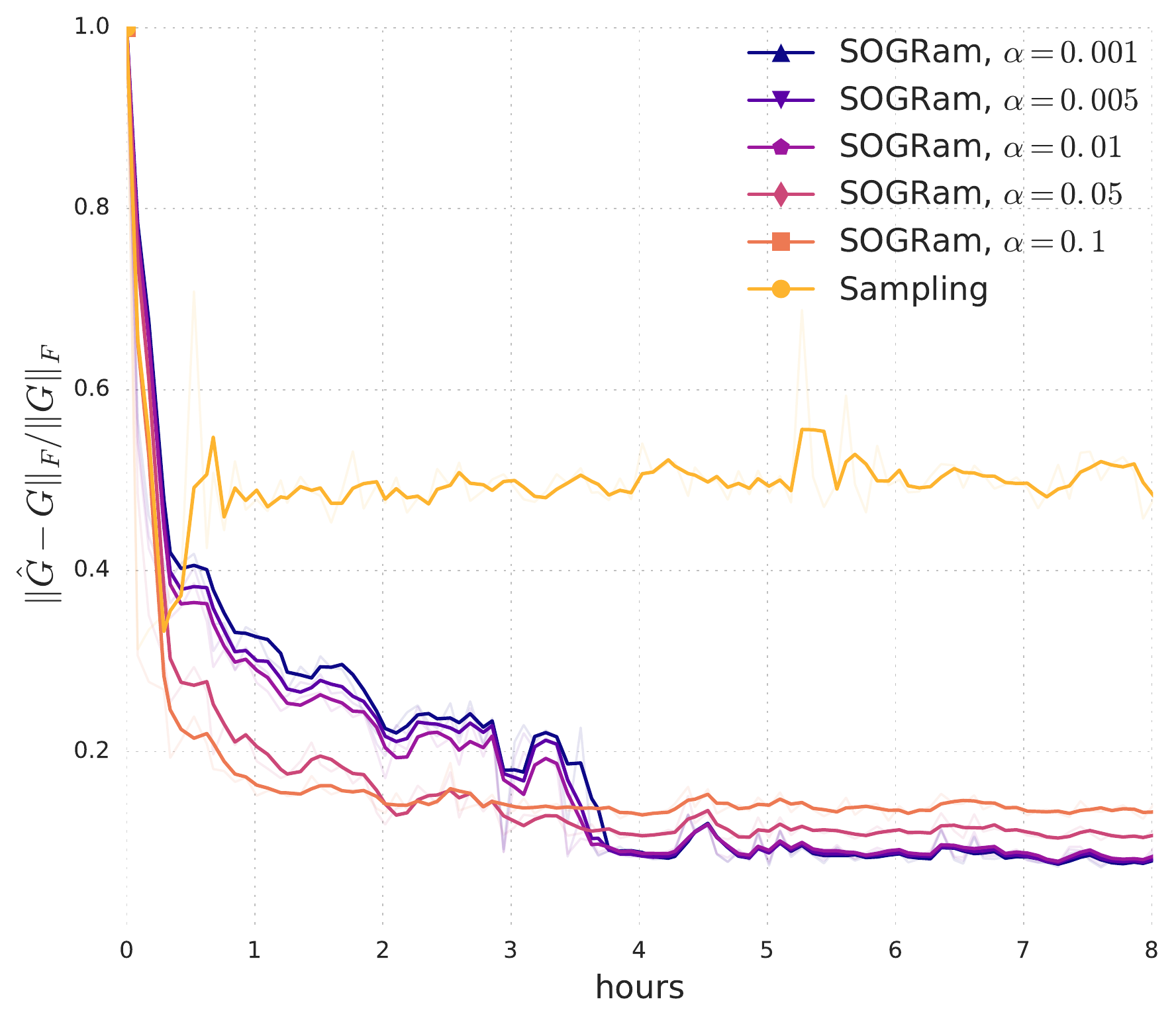}%
\includegraphics[width=.49\textwidth]{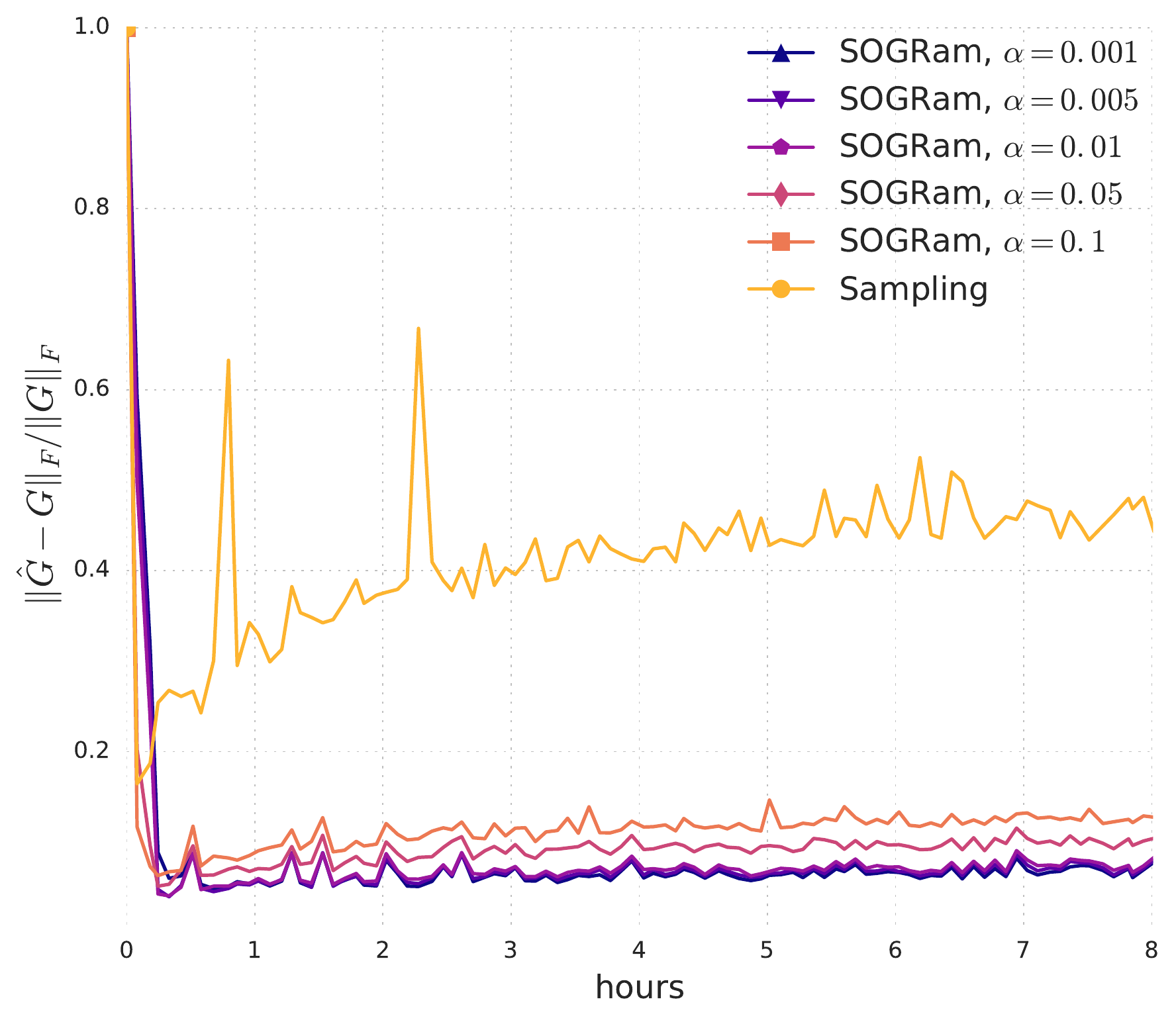}%
   \caption{Gramian estimation error on \texttt{en}, for SOGram with different values of $\alpha$, and different learning rates. The left and right figures correspond respectively to $\eta = 0.01$ and $\eta = 0.002$.}
\label{fig:gram_tradeoff_en}
\end{figure}

%============================================================================================
\section{Experiment on MovieLens data}
\label{app:movielens}
In this section, we report experiments on a regression task on MovieLens.

\paragraph{Dataset} The MovieLens dataset consists of movie ratings given by a set of users. In our notation, the left features $x$ represent a user, the right features $y$ represent an item, and the target similarity is the rating of movie $y$ by user $x$.
The data is partitioned into a training and a validation set using a (80\%-20\%) split. Table~\ref{tbl:movielens-dataset} gives a basic description of the data size. Note that it is comparable to the \texttt{simple} dataset in the Wikipedia experiments.

\begin{table}[h]
\centering
{\small\begin{tabular}{l|r|r|r}
\hline
Dataset & \# users & \# movies & \# ratings \\
 \hline
 \texttt{MovieLens} & 72K & 10K &  10M \\
 \hline
 \end{tabular}}%
  \vspace{.05in}
 \caption{Corpus size of the MovieLens dataset.}\label{tbl:movielens-dataset}
\end{table}

\paragraph{Model} We train a two-tower neural network model, as described in Figure~\ref{fig:model}, where each tower consists of an input layer, a hidden layer, and output embedding dimension $k = 35$. The left tower takes as input a one-hot encoding of a unique user id, and the right tower takes as input one-hot encodings of a unique movie id, the release year of the movie, and a bag-of-words representation of the genres of the movie. These input embeddings are concatenated and used as input to the right tower.

\paragraph{Methods} The model is trained using SOGram with different values of $\alpha$, and sampling as a baseline. We use a learning rate $\eta = 0.05$, and gravity coefficient $\lambda = 1$. We measure mean average precision on the trainig set and validation set, following the same procedure described in Section~\ref{sec:experiments}. The results are given in Figure~\ref{fig:movielens-validation-MAP}.

\begin{figure}[h]
\centering
\includegraphics[width=.49\textwidth]{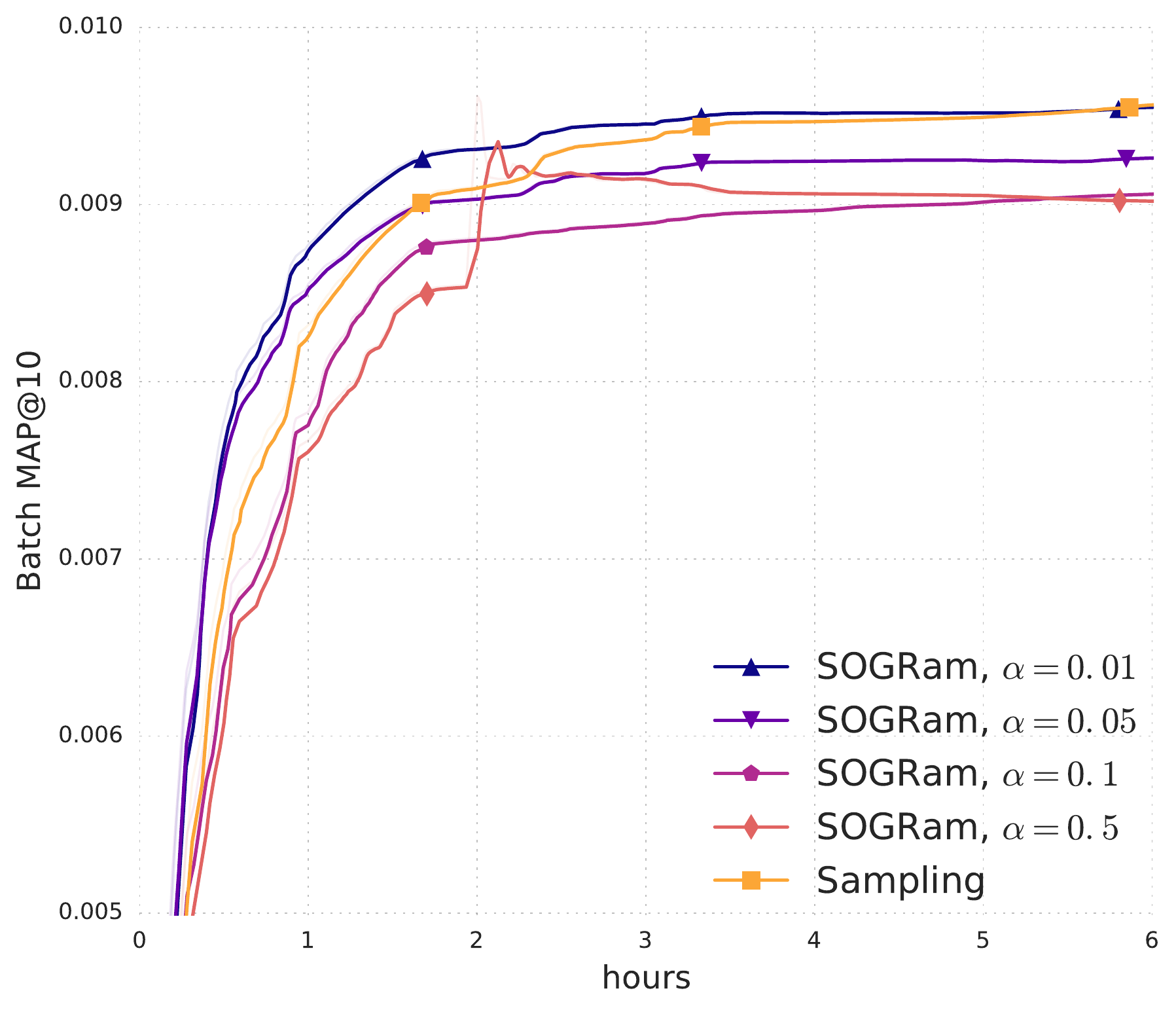}
\includegraphics[width=.49\textwidth]{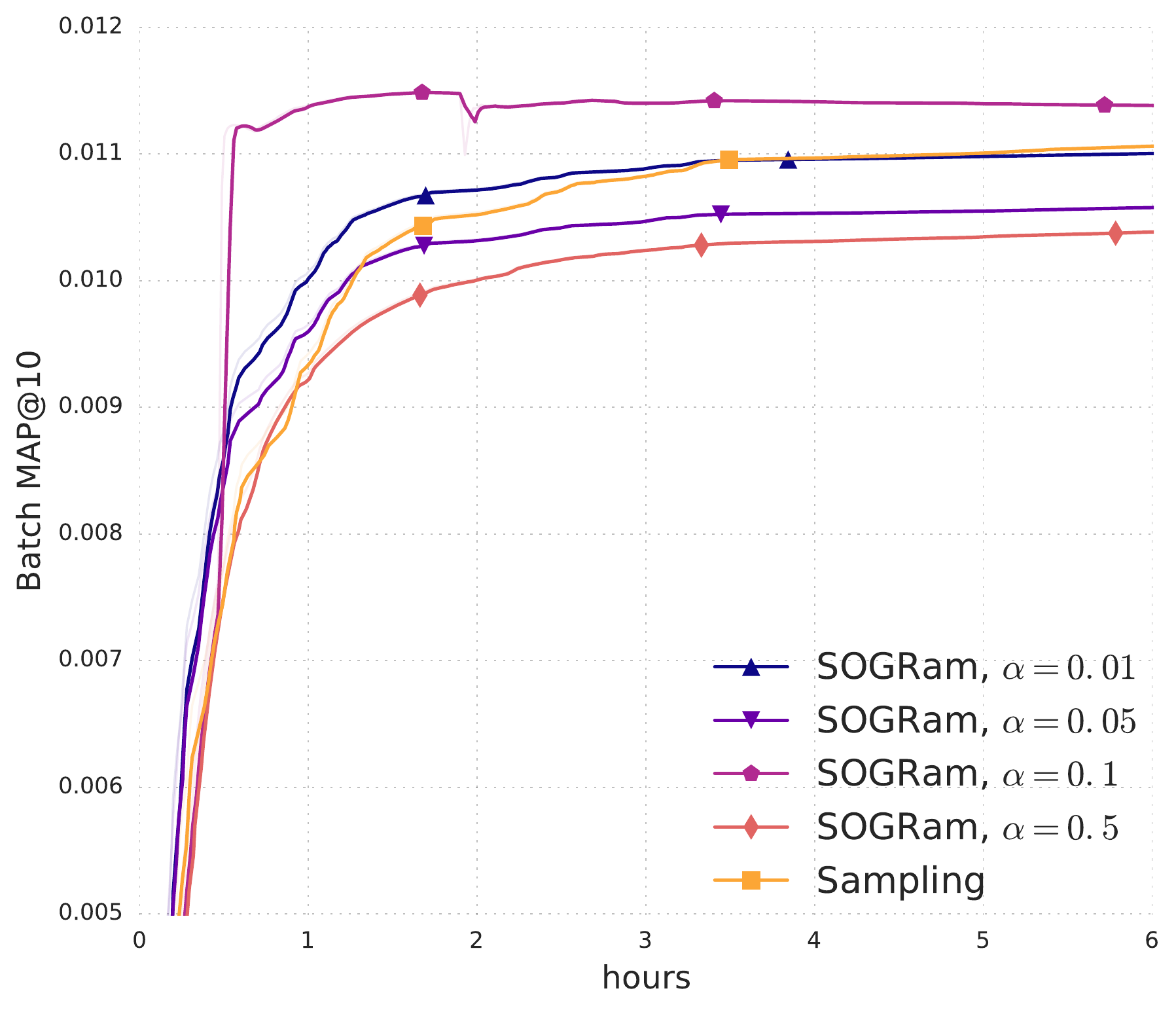}
\caption{Mean average precision at 10 on the training set (left) and the validation set (right), for different methods, on the MovieLens dataset.}
\label{fig:movielens-validation-MAP}
\end{figure}

\paragraph{Results} The results are similar to those reported on the Wikipedia \texttt{simple} dataset, which is comparable in corpus size and number of observations to MovieLens. The best validation mean average precision is achieved by SOGram with $\alpha = 0.1$ (for an improvement of 2.9\% compared to the sampling baseline), despite its poor performance on the training set, which indicates that better estimation of the gravity term $g(\theta)$ induces better regularization. The impact on training speed is also remarkable in this case, SOGram with $\alpha = 0.1$ achieves a better validation performance in under 1 hour of training than the sampling baseline in 6 hours.

%============================================================================================

\end{document}